\documentclass[sigconf]{acmart}

\AtBeginDocument{%
  \providecommand\BibTeX{{%
    \normalfont B\kern-0.5em{\scshape i\kern-0.25em b}\kern-0.8em\TeX}}}

\copyrightyear{2023}
\acmYear{2023}
\setcopyright{acmcopyright}
\acmConference[WSDM '23] {Proceedings of the Sixteenth ACM International Conference on Web Search and Data Mining}{February 27--March 3, 2023}{Singapore, Singapore.}
\acmBooktitle{Proceedings of the Sixteenth ACM International Conference on Web Search and Data Mining (WSDM '23), February 27--March 3, 2023, Singapore, Singapore}
\acmPrice{15.00}
\acmISBN{978-1-4503-9407-9/23/02}
\acmDOI{10.1145/3539597.3570458}

\usepackage{url}
\usepackage{graphicx}
\usepackage{amsmath}
\usepackage{booktabs}
\usepackage{amsmath, amsthm, amssymb, mathtools}
\usepackage{subfigure}
\usepackage{array}
\usepackage{multirow}
\usepackage{hyperref}
\usepackage{booktabs}
\usepackage{threeparttable}
\usepackage[ruled]{algorithm2e}
\usepackage{bm}
\usepackage{enumitem}
\usepackage{dsfont}
  
\usepackage{balance}

\newtheorem{definition}{Definition}

\newtheorem{theorem}{Theorem}

\begin{document}

\title{Unbiased and Efficient Self-Supervised Incremental Contrastive Learning}

\author{Cheng Ji}
\affiliation{%
  \institution{Beihang University}
  \city{Beijing}
  \country{China}}
\email{jicheng@act.buaa.edu.cn}

\author{Jianxin Li}
\affiliation{%
  \institution{Beihang University}
  \city{Beijing}
  \country{China}}
\email{lijx@act.buaa.edu.cn}

\author{Hao Peng}
\affiliation{%
  \institution{Beihang University}
  \city{Beijing}
  \country{China}}
\email{penghao@act.buaa.edu.cn}

\author{Jia Wu}
\affiliation{%
  \institution{Macquarie University}
  \city{Sydney}
  \country{Australia}}
\email{jia.wu@mq.edu.au}

\author{Xingcheng Fu}
\affiliation{%
  \institution{Beihang University}
  \city{Beijing}
  \country{China}}
\email{fuxc@act.buaa.edu.cn}

\author{Qingyun Sun}
\affiliation{%
  \institution{Beihang University}
  \city{Beijing}
  \country{China}}
\email{sunqy@act.buaa.edu.cn}

\author{Philip S. Yu}
\affiliation{%
  \institution{University of
Illinois at Chicago}
  \city{Chicago}
  \country{USA}}
\email{psyu@uic.edu}

\renewcommand{\shortauthors}{Cheng Ji, et al.}

\begin{abstract}
Contrastive Learning (CL) has been proved to be a powerful self-supervised approach for a wide range of domains, including computer vision and graph representation learning. 
However, the incremental learning issue of CL has rarely been studied, which brings the limitation in applying it to real-world applications. 
Contrastive learning identifies the samples with the negative ones from the noise distribution that changes in the incremental scenarios.
Therefore, only fitting the change of data without noise distribution causes bias, and directly retraining results in low efficiency. 
To bridge this research gap, we propose a self-supervised \emph{Incremental Contrastive Learning} (ICL) framework consisting of 
(i) a novel Incremental InfoNCE (NCE-II) loss function by estimating the change of noise distribution for old data to guarantee no bias with respect to the retraining,
(ii) a meta-optimization with deep reinforced Learning Rate Learning (LRL) mechanism which can adaptively learn the learning rate according to the status of the training processes and achieve fast convergence which is critical for incremental learning. 
Theoretically, the proposed ICL is equivalent to retraining, which is based on solid mathematical derivation.
In practice, extensive experiments in different domains demonstrate that, without retraining a new model, ICL achieves up to $16.7\times$ training speedup and $16.8\times$ faster convergence with competitive results.

\end{abstract}




\maketitle

\section{Introduction}\label{sec:introduction}

Contrastive Learning (CL) is a widely used self-supervised learning approach across a wide range of domains, 
such as computer vision (CV)~\cite{he2020momentum,chen2020simple}, 
neural language processing (NLP)~\cite{logeswaran2018efficient,oord2018representation,li2022survey},
and graph representation learning (GRL)~\cite{qiu2020gcc,you2020graph,sun2021sugar}. 
The main idea of contrastive learning is to make representations of similar samples close and distinct samples far away through a noise contrastive estimation (NCE)~\cite{gutmann2010noise,gutmann2012noise} loss function with the given noise distribution~\cite{oord2018representation}. 
Moreover, in real-world application scenarios, online systems are expected to constantly face new data and learn incrementally, which limits the applicability of contrastive learning. The data distribution and noise distribution are incrementally observed and the estimation by NCE is thus biased. Nevertheless, there is little work to study the incremental learning issue of contrastive learning (ICL, Incremental Contrastive Learning).

Incremental learning aims to learn the new data while not forgetting the old data (see the formal definition in Definition~\ref{def:incremental_contrastive_learning}). It requires the learning model to face the stability-plasticity dilemma~\cite{mermillod2013stability} with the following three characteristics: \textbf{(a) stability for the old data,} which means that the model should contain the ability to remember and update the knowledge of the old data, \textbf{(b) plasticity for the new data,} which requires that the model should be able to adapt to the new data, \textbf{(c) efficiency of the training process,} which helps the model update quickly in real-world applications, especially in a streaming environment.

\begin{table}[t]
    \centering
    \caption{Comparison of different strategies.}
    \renewcommand\arraystretch{1.2}
    \begin{threeparttable}
    \begin{tabular}{cccc}
        \toprule
        & \textbf{Stability} & \textbf{Plasticity} & \textbf{Efficiency}\\
        \midrule
        Inference           & $\times$ & $\times$ & $\checkmark$ \\
        Fine-tuning           & $\times$ & $\times$ & $\checkmark$ \\
        Retraining          & $\checkmark$ & $\checkmark$ & $\times$ \\
        \textbf{ICL (Ours)} & $\checkmark$ & $\checkmark$ & $\checkmark$ \\
        \bottomrule
    \end{tabular}
    \begin{tablenotes}
        \item[1] $\checkmark$ : ``unbiased'' or ``high'', $\times$ : ``biased'' or ``low''.
    \end{tablenotes}
    \end{threeparttable}
    \label{tab:1}
\end{table}

The major challenge of contrastive learning in the incremental setting is how to estimate the change of the noise distribution. 
The noise distribution in contrastive learning is used to sample the negative ones and is related to and close to the data distribution~\cite{gutmann2012noise,oord2018representation}. Therefore, as the data change, the noise distribution also changes. 
Based on the trained model, the old data has been learned through NCE by contrasting the negative samples from the original noise distribution. However, there is a bias in the estimation after giving the new data since the noise distribution used for sampling the negative samples has changed. The existing NCE methods cannot be applied to fit the change (we provide the bias analysis in Section~\ref{sec:methodology}).
Regardless of the trained model, retraining seems to be the best way due to no bias in the process. However, it has a serious time-consuming problem.
Therefore, in this paper, we seek to answer the question: \emph{can contrastive learning approaches be able to estimate the old data through an unbiased NCE w.r.t. retraining, while learning the new data efficiently?}

The naive strategies (e.g., inference, fine-tuning, and retraining) fail to answer this question due to the above three requirements of incremental learning as shown in Table~\ref{tab:1}:
\begin{itemize}
    \item \textbf{Poor Stability.} Without a carefully designed NCE strategy, both inference and fine-tuning lead to a biased estimate due to the change in the noise distribution. In addition, directly fine-tuning the model on new data causes catastrophic forgetting~\cite{mccloskey1989catastrophic}.
    \item \textbf{Weak Plasticity.} Inference is not capable of learning the new data and fine-tuning is over-transferring, ignoring the contribution of old data to the noise distribution. Moreover, bias will continue to accumulate in practical online applications.
    \item \textbf{Low Efficiency.} While retraining a new model with all data does not seem to have the above issues, the time cost is unacceptable in real-world applications.
\end{itemize}

However, there is little work on incremental learning in CL, except for applying the replay and distillation technique~\cite{cha2021co,lin2021continual}.
Moreover, existing incremental learning approaches are hard to use for self-supervised CL because they focus mainly on class-incremental learning~\cite{mai2021supervised} and task-incremental learning~\cite{serra2018overcoming}. Without any information of labels and tasks, self-supervised incremental contrastive learning is still facing a research gap.

\textbf{Contributions.} To this end, we propose an unbiased and efficient self-supervised \emph{Incremental Contrastive Learning} (ICL) framework. First, we design an \textbf{I}ncremental \textbf{I}nfo\textbf{NCE} (\textbf{NCE-II}) loss function to fit the change of noise distribution. Furthermore, we accelerate the convergence by a meta-optimization algorithm with a reinforced Learning Rate Learning (LRL) mechanism. Finally, we conduct thorough experiments on four datasets of CV and GRL to show the efficiency and effectiveness of ICL.
The main contributions of this paper are summarized as follows:
\begin{itemize}
    \item Leveraging a new metric for measuring the change of the noise distribution, we design a novel NCE-II loss to theoretically achieve an unbiased ICL with respect to retraining. To our best knowledge, it is the first attempt to study the issue of unbiased incremental learning in self-supervised CL.
    \item We propose a meta-optimization algorithm with LRL to adaptively learn the learning rates according to the status of the training process for fast convergence.
    \item Extensive experiments demonstrate that ICL maintains high efficiency in terms of training time and convergence epoch. In addition, as an unbiased and efficient approach, ICL still achieves competitive results.
\end{itemize}

\section{Background and problem formulation}\label{sec:background}

We first provide the background and problem formulation of contrastive learning and incremental contrastive learning.

\begin{definition}[\textbf{Contrastive Learning}]\label{def:contrastive_learning}
Given the set of input data $X \coloneqq \{x_i \in \mathcal{X}\}_{i=1}^N$ of size $N$ and an encoder $\phi(\cdot) :  \mathcal{X}  \to  \mathcal{Z}$ mapping the inputs from the data space $\mathcal{X}$ into the latent space $\mathcal{Z}$, contrastive learning aims to train the encoder with a contrastive loss function $\mathcal{L}$ designed to identify the positive sample $x_i^+$ for a given $x_i$.
\end{definition}

In this paper, we consider the most commonly used contrastive learning framework~\cite{chen2020simple,he2020momentum,logeswaran2018efficient,you2020graph} with the \textbf{I}nfo\textbf{NCE} (denoted as \textbf{NCE-I} in this paper) as the loss function $\mathcal{L}$.

\begin{definition}[\textbf{InfoNCE (NCE-I)}]\label{def:nce}
Given the set of input data $X \coloneqq \{x_i \in \mathcal{X}\}_{i=1}^N$ and an encoder $\phi(\cdot)$, the InfoNCE (NCE-I) loss is defined as:
\begin{equation}\label{eq:infonce}
    \mathcal{L}_i^{I,X} = - \log \frac{f(x_i,x_i^+)}{f(x_i,x_i^+) + K\mathbb{E}_{x_i^- \sim p_n^X} f(x_i,x_i^-)},
\end{equation}
where $f(\cdot,\cdot) \coloneqq \exp(sim(\phi(\cdot),\phi(\cdot))/\tau)$ with $sim(\cdot, \cdot)$ as the similarity measurement and $\tau$ as the temperature parameter,  $x_i^+$ represents the positive sample generated by a data augmentation module, $x_i^-$ is one of the negative samples from a noise distribution, and $K$ is a hyperparameter representing the ratio of negative samples to positive samples. 
\end{definition}

\textbf{Noise Distribution.}
The InfoNCE loss follows the noise contrastive estimation (NCE) principle~\cite{gutmann2010noise,gutmann2012noise}, which learns to deduce the properties of data $X$ by comparing the difference to the reference (noise) data $Y$. The noise data $Y$ is an i.i.d. sample $\{y_1,y_2,\dots,y_K\}$ from a random variable with noise distribution $p_n$. In this paper, the negative ones are from the data distribution by a uniform sampling which can be approximated by Monte Carlo sampling~\cite{shapiro2003monte}. We randomly sample a mini-batch of $K + 1$ inputs and treat the others as the negative samples for each one.

\begin{definition}[\textbf{Incremental \, Contrastive \, Learning}]\label{def:incremental_contrastive_learning}
Given an encoder $\phi(\cdot)$ trained on old data $X \coloneqq \{x_i \in \mathcal{X}\}_{i=1}^N$ with a contrastive learning approach and the new data $\Delta X \coloneqq \{x_{N+i} \in \mathcal{X}\}_{i=1}^{\Delta N}$ that has not been observed, the incremental contrastive learning aims to refine the encoder for adapting to the new data without forgetting the knowledge of old data, i.e., stability and plasticity. 
\end{definition}

\textbf{Noise Distribution Change.}
As mentioned above, the negative ones in contrastive learning are randomly sampled from the dataset which changes in the incremental setting. That is, the noise distribution changes.

Finally, we denote $X' \coloneqq \{x_i \in X \cup \Delta X\}$ as all data.
For generality, we focus on one incremental step and note that the old data $X$ represent the data used for training the encoder before and preserved by the online system through a memory queue or replay technique.
Therefore, the study in this paper can be conveniently implemented to a number of contrastive learning and incremental learning methods.

\section{Methodology}\label{sec:methodology}
In this section, we propose an unbiased and efficient self-supervised \emph{Incremental Contrastive Learning} (ICL) framework.

\subsection{Overall Framework}\label{sec:framework}
The proposed ICL framework consists of the following components:
\begin{enumerate}
    \item \textbf{Augmentation.} Given each input $x_i$ (e.g., an image or a graph), two augmentations $q_1(\cdot|x_i)$  and $q_2(\cdot|x_i)$ are applied to $x_i$ to obtain a positive pair $(x_i,x_i^+)$. For different domains of datasets (i.e., CV and GRL), different augmentation strategies are applied (Section~\ref{exp_deatils}).
    \item \textbf{Encoder.} An encoder $\phi(\cdot) :  \mathcal{X}  \to  \mathcal{Z}$ is used to learn the latent representations for each positive pair $(x_i,x_i^+)$. Specifically, a ResNet-18~\cite{he2016deep} and a GCN~\cite{kipf2016semi} are used for CV an GRL respectively.
    \item \textbf{Incremental InfoNCE (NCE-II).} To resolve the problems mentioned in Section~\ref{sec:introduction}, we design a novel loss function, named \textbf{NCE-II}, to eliminate the bias caused by the change of noise distribution. Furthermore, the proposed NCE-II maintains equivalence with retraining, and the error bound of final empirical risk tends to zero (Section~\ref{sec:nce-ii}).
    \item \textbf{Meta-optimization with Learning Rate Learning (LRL).} In order to further improve the efficiency of the training process, which is crucial for incremental learning, we proposed a new meta-learning optimization algorithm with a reinforced learning rate learning mechanism for fast convergence (Section~\ref{sec:optimization}).
\end{enumerate}
In the following sections, we introduce the proposed objective function (NCE-II) and optimization algorithm (meta-optimization with LRL).

\subsection{Objective Function: Unbiased Estimation}\label{sec:nce-ii}

We split the final objective function into two parts: one for old data and the other for new data.
The new data $\Delta X$ can be learned by NCE-I with the noise distribution $p_n^{X'}$:
\begin{equation}\label{eq:infonce_for_old}
    \mathcal{L}_i^{I,X'} = - \log \frac{f(x_i,x_i^+)}{f(x_i,x_i^+) + K\mathbb{E}_{x_i^- \sim p_n^{X'}} f(x_i,x_i^-)}.
\end{equation}

However, for stability as discussed in Section~\ref{sec:introduction}, InfoNCE loss cannot serve as the objective function for the old data. Given the encoder trained with the old data $X$, it is necessary to find an optimization method for $X$, so that the entire training process (including the incremental learning phase) with $X$ is unbiased.

\subsubsection{Motivation: Change of Noise Distribution}
In order to eliminate the deviation caused by the change of the noise distribution in the learning of the old data,  we first propose a metric for measuring the change ratio of the noise distribution. 

To explore how noise distribution works in contrastive learning, we first rewrite the InfoNCE loss into the form of softmax-based categorical cross-entropy:
\begin{align}
    \mathcal{L}_i^{I,X} 
    &= \sum_{x_j \in X} \mathds{1}_{i=j} \cdot - \log \left( \frac{f(x_i,x_j^+)}{\sum_{x_k \in X} f(x_i,x_k^+)} \right)\\
    &= \sum_{x_j \in X} \mathds{1}_{i=j} \cdot - \log \left( \text{softmax}_X (f(x_i,x_j^+)) \right). \label{eq:softmax}
\end{align}
The softmax term in Eq.~\eqref{eq:softmax} represents the final prediction result of the model, in which the sum of predicted probabilities for all classes changes with the noise distribution, that is, $\sum_{x_k \in X} f(x_i,x_k^+) \to \sum_{x_{k} \in {X \cup \Delta  X}} f(x_i,x_k^+)$. 
Therefore, we propose to use the ratio of the sum of predicted probabilities for new classes to the one for old classes.




\begin{definition}[\textbf{Change Ratio of Noise Distribution}]\label{def:change_rate}
Given the noise distribution $p_n^X$ of old data $X$ and $p_n^{\Delta  X}$ of new data $\Delta  X$, the change ratio is represented by 
\begin{align}
    r_{i}^{X \to \Delta  X} 
    &= \frac{\sum_{x_j \in \Delta  X \cup \{x_i\}} f(x_i,x_j^+)}{\sum_{x_k \in X} f(x_i,x_k^+)}\\
    &= \frac{f(x_i,x_i^+)  +  K\mathbb{E}_{x_i^- \sim p_n^{\Delta  X}}  f(x_i,x_i^-)}{f(x_i,x_i^+)  +  K\mathbb{E}_{x_i^- \sim p_n^X}  f(x_i,x_i^-)}. \label{eq:change_ratio}
\end{align}
\end{definition}

With the proposed change ratio in Eq.~\eqref{eq:change_ratio}, we can measure how much the noise distribution changes. If the new data maintain the same noise distribution as the old data (i.e., $p_n^{\Delta  X}=p_n^X$), the change ratio $r_{i}^{X \to \Delta  X}$ equals $1$. The deviation of $r_{i}^{X \to \Delta  X}$ from 1 reflects the degree of change of noise distribution.

\subsubsection{Objective Function: NCE-II Loss}
Next, we design a novel contrastive loss function, leveraging the change ratio of the noise distribution.

\begin{definition}[\textbf{Incremental InfoNCE (NCE-II)}]\label{def:incremental_infonce}
Given the old data $X$ and new data $\Delta X$, the loss function of the incremental contrastive learning for old data is defined as
\begin{align}
    \mathcal{L}^{I I}_i 
    & =   \log  \left( \alpha  \cdot  r_{i}^{X \to \Delta  X}   +  (1 - \alpha)  \cdot  1  \right) \\
    & =   \log  \left( \alpha  \cdot  \frac{f(x_i,x_i^+)  +  K\mathbb{E}_{x_i^- \sim p_n^{\Delta  X}}  f(x_i,x_i^-)}{f(x_i,x_i^+)  +  K\mathbb{E}_{x_i^- \sim p_n^X}  f(x_i,x_i^-)}   +  (1 - \alpha)  \cdot  1  \right),\label{eq:incremental_infonce}
\end{align}
where constant $\alpha=\frac{\Delta N}{N + \Delta N} \in [0, 1)$ represents the growth ratio of the data.
\end{definition}

In NCE-II loss, the change ratio for each sample $r_{i}^{\Delta  X \to X}$ is weighted by a coefficient $\alpha$, which reflects the increment of the data. Specifically, when there is no new data (i.e., $\alpha = 0$) 
or the noise distribution remains the same (i.e., $r_{i,\Delta X \to X}$=1), the NCE-II equals $0$, which is in line with the intuition. 

\subsubsection{No Bias: Equivalence with Retraining} 
Given NCE-II, the encoder is trained on the old data $X$ with two contrastive loss functions (Eq.\eqref{eq:infonce_for_old} and Eq.\eqref{eq:incremental_infonce}), the sum of which is equivalent to the loss function of retraining on all data $X'$.

\begin{theorem}\label{th:equivalence}
For old data $x_i  \in  X$, the NCE-II with new data $\Delta X$ plus InfoNCE only with $X$ is equivalent to the one with all data $X'$, i.e., $\mathcal{L}_i^{I,X'} = \mathcal{L}_i^{I,X} + \mathcal{L}_i^{I I}$.
\end{theorem}
\begin{proof}
The proof is given in the online repositories\footnote{Proofs are at \url{https://github.com/RingBDStack/ICL-Incremental-InfoNCE}.}.
\end{proof}

Finally, the objective function of ICL is defined as
\begin{equation}\label{eq:objective}
    \mathcal{L} = \sum_{x_i \in X} \mathcal{L}_i^{I I} + \sum_{x_i \in \Delta X} \mathcal{L}_i^{I,X'},
\end{equation}
which is an unbiased estimation of both old data and new data w.r.t. retraining. In Table~\ref{tab:bias}, we provide the comparison of bias between different strategies to show the superiority of the proposed method.

\begin{table}[t]
    \centering
    \caption{Bias of different strategies.}
    \renewcommand\arraystretch{1.2}
    \begin{tabular}{ccc}
        \toprule
        & \textbf{Old Data} & \textbf{New Data} \\
        \midrule
        Inference           & $\log r_{i}^{X \to X'}$ & $\mathcal{L}_i^{I,X'}$ \\
        Fine-tuning           & $\log r_{i}^{X \to X'}$ & $\log r_{i}^{\Delta  X \to X'}$ \\
        Retraining          & 0 & 0 \\
        \textbf{ICL (Ours)} & 0 & 0 \\
        \bottomrule
    \end{tabular}
    \label{tab:bias}
\end{table}

\subsubsection{Bound Analysis.} 
Although the proposed loss function NCE-II is equal to the difference of InfoNCE with old data and all data, the final empirical risk $\mathcal{R}$ during the entire training process is nevertheless different due to the change of data distribution. Therefore, we provide the bound analysis to guarantee the correctness of the proposed NCE-II.

\begin{theorem}\label{th:bound}
The difference between the empirical risk of the method with the proposed NCE-II and the retraining with InfoNCE throughout the training process is bounded by $\alpha \mathcal{R}^{old}$, where $\mathcal{R}^{old}=\frac{1}{N}\sum_{i=1}^N \mathcal{L}_i^{I,old}$ approaches zero as $\mathcal{L}_i^{I,old}$ is minimized after the previous training process on old data $X$ and the growth ratio $\alpha \in [\,0,1)$. Then we have $\alpha \mathcal{R}^{old} \to 0$.
\end{theorem}
\begin{proof}
The proof is given in the online repositories.
\end{proof}

\subsubsection{Complexity Analysis.} 
Given the number of epochs $\Delta T$ in the incremental learning process, the time complexity of the proposed method is $\mathcal{O}(\Delta T((2K + 1)N+(K + 1)\Delta N)C)$, where $C$ represents for the time consuming of the encoder and function $f(\cdot,\cdot)$. As the parameters of the encoder have been optimized, the number of epochs is much smaller than the one of the retraining process. We further guarantee it in the next section. 

\begin{figure*}[t]
    \centering
    \includegraphics[width=0.85\textwidth]{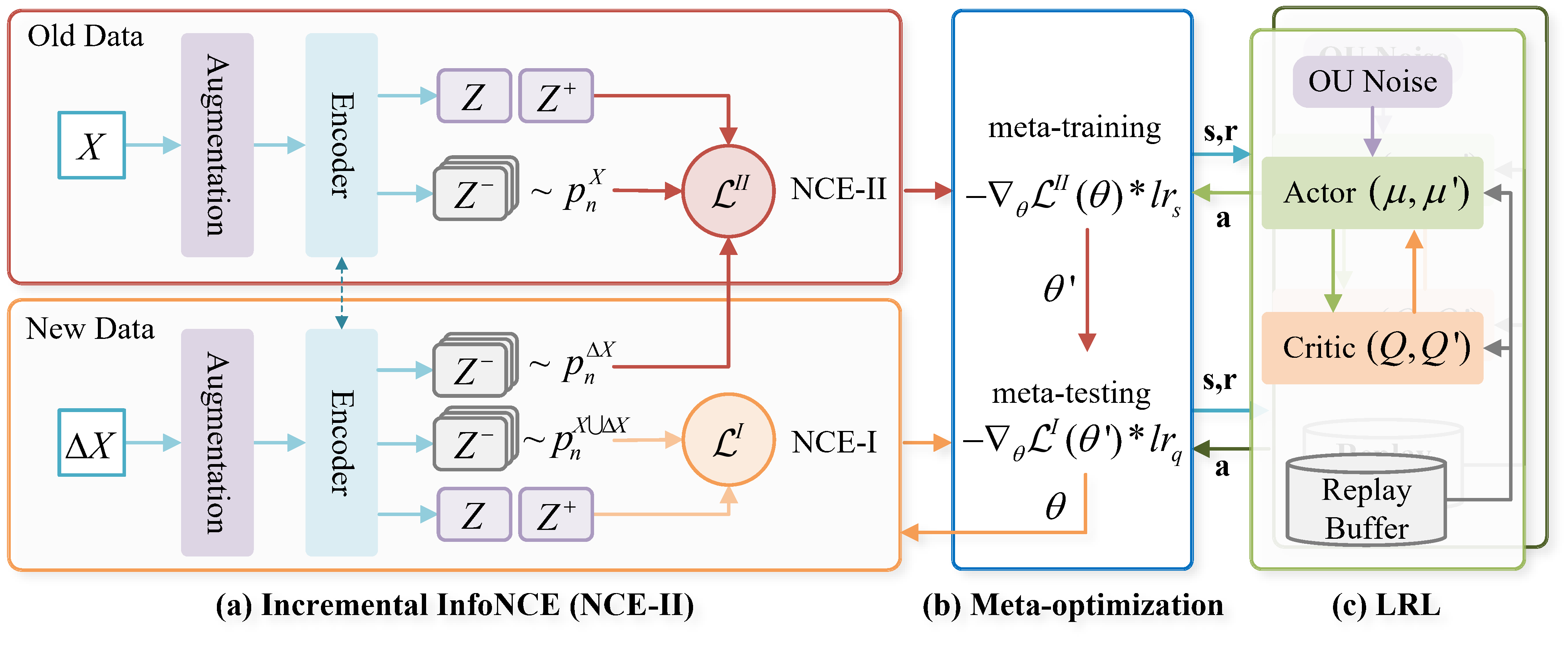}
    \caption{An overview of ICL. 
    (a) The proposed Incremental InfoNCE (NCE-II) to fit the change of noise distribution $p_n$. (b) The proposed meta-optimization where the old data $X$ is treated as the support set in the meta-training stage and the new data $\Delta X$ is treated as the query set in the meta-testing stage. (c) The proposed Learning Rate Learning (LRL) mechanism which uses the current loss as state and loss decrement as reward to generate the learning rates ($lr_s$ and $lr_q$) as actions for meta-optimization.
    }\label{fg:framework}
\end{figure*}

\subsection{Optimization: Efficient Adaption}\label{sec:optimization}
In order to further improve the efficiency of the proposed method and ensure that the number of epochs in the incremental learning process is smaller, we next propose a meta-optimization algorithm and a Learning Rate Learning (LRL) mechanism by reinforcement learning (RL) to quickly adapt to the new data.

Formally, the optimization of the encoder is defined as
\begin{equation}\label{eq:optimization}
    \theta_{t+1} \gets \theta_t - lr * \nabla_\theta \mathcal{L} (\phi(X;\theta_t)),
\end{equation}
where $lr$ is the learning rate and $ \nabla_\theta \mathcal{L} (\phi(X;\theta_t))$ is the gradient in the time step $t$.

\subsubsection{Meta-optimization}\label{sec:meta}
For plasticity, we treat the optimization of ICL for new data as transfer learning (the new data may contain new classes or from different domains) and propose a meta-optimization strategy for fast adaption, which is formulated as a form of model agnostic meta-learning (MAML)~\cite{finn2017model}.

We treat the old data $X$ as the support set and the new data $\Delta X$ as the query set. In the meta-training stage, we calculate the loss on the support set by Eq.\eqref{eq:incremental_infonce} and gain the new parameters $\theta'$ in a few gradient descent steps:
\begin{equation}\label{eq:update_parameters}
    \theta' = \theta - lr_s * \frac{\partial \sum_{x_i \in X} \mathcal{L}_\theta^{I I} (\phi(x_i;\theta))}{\partial \theta},
\end{equation}
where $lr_s$ is the learning rate of the meta-training process on the support set $X$. For fairness, the number of steps is set to $\max\{\lceil (1-\alpha)/\alpha \rceil, 1\}$ to ensure that the total number of samples trained in one epoch and the size of all data are as equal as possible, under the premise that there is at least one support sample for each query.

In the meta-testing stage, after obtaining the adapted parameters $\theta'$, we update the parameters $\theta$ on the query set with
\begin{equation}\label{eq:optimize_parameters}
    \theta \gets \theta - lr_q * \frac{\partial \sum_{x_i \in \Delta X} \mathcal{L}_\theta^{I,X'} (\phi(x_i;\theta'))}{\partial \theta},
\end{equation}
where $lr_q$ is the learning rate of the meta-testing process on the query set $\Delta X$.



\subsubsection{Learning Rate Learning (LRL)}\label{sec:lrl}
The learning rate for gradient descent is commonly set as a hyperparameter and needs to be manually searched for the optimal value. However, it is a challenge to determine the values of two learning rates ($lr_q$ and $lr_s$) in our proposed meta-optimization. It is (a) unreasonable to simply set the two values to the same or use the previous one because the two learning rates are used for different training processes and data, (b) hard to search for the optimal value due to the large searching space, and (c) the change rules of learning rates should be automatically learned according to the status of the training environment instead of applying the human-designed ones. Therefore, we aim to find a solution that can adaptively control learning rates while speeding up the convergence process.

We proposed a Learning rate Learning (LRL) mechanism to learn the two learning rates ($lr_s$ and $lr_q$) adaptively according to the current state of the encoder in a reinforcement learning approach~\cite{kaelbling1996reinforcement,xu2017reinforcement,li2021reinforcement}. 
In our reinforcement learning setting, we propose to use a function $\mathcal{F}$ to generate the state $s_t$ which encodes the observation of the training process (the inputs $X$ and the parameters $\theta_t$) at the time step $t$:
\begin{equation}
    s_t = \mathcal{F}(X, \theta_t).
\end{equation}
The action $a_t$ is defined as the learned learning rate based on the state $s_t$ and $a_t \in \mathbb{R}$ is a continuous value. 

To resolve the continuous action space issue, \cite{xu2017reinforcement} uses an actor-critic algorithm~\cite{sutton1999policy,silver2014deterministic} to generate the learning rate, employing two neural networks: the actor network for learning rate generation and the critic network for criticizing the actions. However, because the inputs of the actor network are ordered and related (i.e., not i.i.d.), the reinforcement learning process is therefore unstable and problematic~\cite{mnih2013playing}. Therefore, we propose to use the Deep Deterministic Policy Gradient (DDPG) to learn the two learning rates respectively, which contains the experience replay mechanism to resolve the i.i.d. issue and the separate target network mechanism for a stable training~\cite{lillicrap2016continuous}. 

In general, DDPG networks $\mathcal{D}=\{\mu, Q, \mu', Q'\}$ consist of the (online) actor network $\mu(s|\theta^\mu)$ and (online) critic network $Q(s,a|\theta^Q)$, and the target actor network $\mu'$ and the target critic network $Q'$ with the same initial weights as the online ones. We first select an action $a_t$ in step $t$ as
\begin{equation}\label{eq:generate_learning_rate}
    a_t = \mu(s_t|\theta^\mu) + \mathcal{N}_t,
\end{equation}
where $\mathcal{N}$ is the Ornstein-Uhlenbeck (OU) process~\cite{uhlenbeck1930theory} to generate noise for action exploration. Next, the action $a_t$ is executed (updating encoder's parameters with $a_t$ as the learning rate) and the new state $s_{t+1}$ and reward $r_t$ are observed. As the aim of LRL is to accelerate the convergence, we define the reward $r_t$ as the loss decrement:
\begin{equation}\label{eq:reward}
    r_t = \mathcal{L}_t - \mathcal{L}_{t+1}.
\end{equation}
Finally, we optimize the the critic network using Temporal-Difference (TD) learning with a replay buffer and actor network with the chain rule~\cite{lillicrap2016continuous}. The target networks are updated by a momentum mechanism. Please refer to the online repositories for more details.



\section{Experiment}\label{sec:exp}
Extensive experiments across two domains, computer vision and graph representation learning, are conducted to demonstrate the high efficiency of the proposed unbiased ICL\footnote{Code is available at \url{https://github.com/RingBDStack/ICL-Incremental-InfoNCE}.} framework. In addition, ICL also achieves competitive results. We further provide the incremental setting analysis and ablation study.

\begin{table*}[t]
    \centering
    \caption{Classification results of running time and convergence epoch on CV datasets. (\textbf{bold}: best; \underline{underlined}: runner-up)}
    \resizebox{\textwidth}{!}{\begin{tabular}{lrrrrrr|rrrrrr|cc}
        \toprule
        \textbf{Dataset} &  
        \multicolumn{6}{c}{\textbf{ImageNet-2}} & 
        \multicolumn{6}{c}{\textbf{MNIST-2}} & 
        \multicolumn{2}{c}{\multirow{2}{*}{\textbf{Avg. Rank}}} \\
        \textbf{Growth Rate} &
        \multicolumn{2}{c}{\textbf{$\alpha=0.3$}} & 
        \multicolumn{2}{c}{\textbf{$\alpha=0.5$}} & 
        \multicolumn{2}{c}{\textbf{$\alpha=0.7$}} &
        \multicolumn{2}{c}{\textbf{$\alpha=0.3$}} & 
        \multicolumn{2}{c}{\textbf{$\alpha=0.5$}} & 
        \multicolumn{2}{c}{\textbf{$\alpha=0.7$}} \\
        \cmidrule(r){2-15}
        \textbf{Metric} & Time & Epoch  & Time & Epoch  & Time & Epoch  & Time & Epoch  & Time & Epoch  & Time & Epoch & Time & Epoch \\
        \midrule
        Retraining & 1.0x & 1.0x  & 1.0x & 1.0x  & 1.0x & 1.0x  & 1.0x & 1.0x  & 1.0x & 1.0x  & 1.0x & 1.0x & - & -  \\
        Fine-tuning & \underline{10.7x} & 3.2x  & 3.9x & 2.0x  & 2.2x & 1.6x  & \underline{12.3x} & \underline{3.6x}  & \underline{5.5x} & 2.6x  & \underline{3.8x} & \underline{2.7x} & \underline{2.3} & \underline{3.6} \\
        Replay & 3.4x & 2.0x & 2.0x & 1.6x & \underline{3.4x} & 2.7x & 2.8x & 2.6x & 3.5x & 2.5x & 3.2x & 2.4x & 4.3 & 4.5\\
        Distillation & 6.9x & 2.9x & 2.1x & 1.4x & 2.2x & 1.5x & 7.8x & 3.0x & 4.4x & \underline{2.8x} & 3.1x & 2.2x & 4.0 & 4.8\\
        \midrule
        \textbf{ICL (Ours)}  & \textbf{11.2x} & \textbf{16.8x}  & \textbf{6.2x} & \textbf{13.2x} & \textbf{5.9x} & \textbf{12.7x}  & \textbf{14.2x} & \textbf{10.6x}  & \textbf{16.7x} & \textbf{12.7x}  & \textbf{6.4x} & \textbf{7.1x}  & \textbf{1.0} & \textbf{1.0}\\
        ICL (w/o LRL)        & 9.6x  & 14.7x  & 3.4x & 5.3x  & 1.1x & 2.2x   & 1.0x & 1.0x  & 2.4x & 2.5x  & 1.5x & 2.3x & 5.2 & 4.0 \\
        ICL (w/o M)          & 9.2x  & 15.7x  & \underline{4.8x} & \underline{8.1x} & 1.7x & \underline{3.8x}   & 1.1x & 1.3x  & 1.2x & 1.5x  & 1.0x & 1.2x & 5.3 & 4.5 \\
        ICL (w/o M+LRL)      & 9.8x  & \underline{16.4x}  & 1.6x & 2.7x  & 1.0x & 1.3x   & 1.7x & 1.7x  & 1.3x & 1.3x  & 1.2x & 1.3x & 5.7 & 5.3 \\
        \bottomrule
    \end{tabular}}\label{tab:classification_cv_time}
    
\end{table*}

\begin{table*}[t]
    \centering
    \caption{Classification results of running time and convergence epoch on GRL datasets. (\textbf{bold}: best; \underline{underlined}: runner-up)}
    \resizebox{\textwidth}{!}{\begin{tabular}{lrrrrrr|rrrrrr|cc}
        \toprule
        \textbf{Dataset} &  
        \multicolumn{6}{c}{\textbf{PROTEINS}} & 
        \multicolumn{6}{c}{\textbf{REDDIT}} & 
        \multicolumn{2}{c}{\multirow{2}{*}{\textbf{Avg. Rank}}} \\
        \textbf{Growth Rate} &
        \multicolumn{2}{c}{\textbf{$\alpha=0.3$}} & 
        \multicolumn{2}{c}{\textbf{$\alpha=0.5$}} & 
        \multicolumn{2}{c}{\textbf{$\alpha=0.7$}} &
        \multicolumn{2}{c}{\textbf{$\alpha=0.3$}} & 
        \multicolumn{2}{c}{\textbf{$\alpha=0.5$}} & 
        \multicolumn{2}{c}{\textbf{$\alpha=0.7$}} \\
        \cmidrule(r){2-15}
        \textbf{Metric} & Time & Epoch  & Time & Epoch  & Time & Epoch  & Time & Epoch  & Time & Epoch  & Time & Epoch & Time & Epoch  \\
        \midrule
        Retraining & 1.0x & 1.0x  & 1.0x & 1.0x  & 1.0x & 1.0x  & 1.0x & 1.0x  & 1.0x & 1.0x  & 1.0x & 1.0x & - & -  \\
        Fine-tuning & \textbf{25.7x} & 8.0x & \textbf{8.1x} & \underline{3.9x}  & \textbf{2.7x} & 1.9x  & \textbf{14.4x} & 1.1x  & \textbf{5.5x} & \underline{2.9x}  & 1.2x & 0.9x & \textbf{1.7} & 3.8 \\
        Replay & 2.9x & 2.1x & 1.6x & 1.0x & 1.4x & 1.4x & 2.9x & 2.1x & 2.0x & 1.4x & 1.3x & 1.1x & 4.8 & 5.8\\
        Distillation & \underline{18.2x} & 7.3x & 5.2x & 2.7x & 0.9x & 0.8x & 2.4x & 1.5x & 2.7x & 1.6x & 1.1x & 0.9x & 4.8 & 5.3 \\
        \midrule
        \textbf{ICL (Ours)}  & 10.1x & \underline{10.5x} & \underline{5.8x} & \textbf{6.1x}  & \underline{2.6x} & \textbf{2.7x}  & 2.7x & 2.5x  & \underline{2.9x} & \textbf{3.1x}  & 2.3x & \underline{3.2x} & \underline{3.0} & \textbf{1.8} \\
        ICL (w/o LRL)        & 3.8x & \textbf{14.6x} & 1.6x & 1.7x  & 1.0x & 1.1x  & 3.4x & \underline{3.3x}  & 2.7x & \underline{2.9x}  & \textbf{2.6x} & \textbf{3.6x} & 3.8 & \underline{3.0} \\
        ICL (w/o M)          & 2.6x  & 3.6x  & 2.1x & 2.9x  & 1.4x & \underline{2.1x}  & 3.0x & 3.2x  & 1.4x & 1.6x  & 1.2x & 1.3x & 5.0 & 3.7  \\
        ICL (w/o M+LRL)      & 2.5x  & 2.9x  & 1.9x & 2.0x  & 1.1x & 1.2x  & \underline{5.9x} & \textbf{5.8x}  & 2.2x & 2.4x  & \underline{2.5x} & 2.5x & 4.1 & 4.0 \\
        \bottomrule
    \end{tabular}}\label{tab:classification_grl_time}
\end{table*}

\subsection{Experimental Settings}
\subsubsection{Datasets}
We use four benchmark datasets across computer vision (CV) and graph representation learning (GRL). 
For CV, we use the following two datasets: (a) ImageNet is a benchmark dataset consisting of around 1.28 million images~\cite{deng2009imagenet}. In order to reduce the influence of additional factors for an accurate efficiency evaluation, we extract a subset with 2 classes, named ImageNet-2, to manage that the entire training process can be completed on 1 GPU within 6 hours. (b) MNIST is a handwritten digit dataset~\cite{lecun1998gradient}. Similarly, we extract a subset with 2 classes, named MNIST-2.
For GRL, we use the following two datasets: (a) PROTEINS is a bioinformatics dataset with 1,113 molecular graphs~\cite{morris2020tudataset}. (b) REDDIT is a social network dataset consisting of 2,000 graphs in 2 classes~\cite{morris2020tudataset}.
We use Local Degree Profile (LDP) algorithm~\cite{cai2018simple} to generate node features.

\subsubsection{Baselines and ICL Variants}
We compare four baselines including the simple strategies (retraining and fine-tuning) and commonly used techniques (replay and distillation). \begin{itemize}
    \item \textbf{Retraining} uses both old data and new data to train a new encoder. Since the aim of this study is to find an approach for unbiased estimation, we first compare ICL with retraining which serves as the baseline of efficiency and the upper bound of effectiveness. 
    \item \textbf{Fine-tuning} updates the parameters of the pre-trained encoder using the new data. 
\end{itemize}
Furthermore, since ICL is under the self-supervised setting which lacks related research, we survey the replay and distillation technique used in incremental learning~\cite{rebuffi2017icarl,castro2018end,fang2020seed,cha2021co,lin2021continual}.
\begin{itemize}
    \item \textbf{Replay} uses the previously seen data (partially stored old data) and the new data to train the encoder for avoiding catastrophic forgetting.
classes, 
    \item \textbf{Distillation} learns a more compact encoder from the old one to prevent over-drift of representations from previous data when learning new ones.  
\end{itemize}
For fairness, we use the exact same network architecture for all methods. 

We further introduce the following three variants of ICL to verify the effectiveness of the components (LRL and meta-optimization):
\begin{itemize}
    \item ICL without LRL mechanism (w/o LRL), i.e., with NCE-II and meta-optimization.
    \item ICL without meta-optimization (w/o M), i.e., with NCE-II and LRL only generating one learning rate.
    \item ICL without meta-optimization and LRL mechanism (w/o M+LRL), i.e., only with NCE-II.
\end{itemize}

\subsubsection{Implementation Details}\label{exp_deatils}

We adopt a two-stage scheme~\cite{you2020graph,li2021rwne,li2022aiqoser}. In the training process, for fairness, all of the methods use the same single encoder to generate the representations. In the testing process, an extra SVM classifier~\cite{you2020graph} is trained with the fixed embedding for evaluation. In order to evaluate the efficiency of the methods more accurately and reduce the influence of external factors, we conduct each experiment on one single Nvidia V100 32G GPU for 5 times independently. We record the mean value as the final accuracy for each experiment and omit the standard deviations (all deviations are around 0.01).
(a) The encoder we have used for CV is ResNet-18~\cite{he2016deep} and a two-layer Graph Convolution Network (GCN)~\cite{kipf2016semi} with 32 hidden units is used for GRL. 
(b) We choose \{random resizing, 224×224-pixel cropping, random color jittering, random grayscale conversion, random horizontal flip\}~\cite{wu2018unsupervised,he2020momentum} as the data augmentation for images, and randomly sample one from \{random node dropping, random node attribute masking, random subgraph selection\}~\cite{you2020graph} for GRL. 
(c) The similarity measurement we have used is the cosine similarity function $sim(z_i,z_j) \coloneqq z_i \cdot z_j / ||z_i|| \cdot ||z_j||$. (d) The state generated by function $\mathcal{F}$ in LRL is defined as the average loss of the current training process on the mini-batch data. 
(e) A two-layer LSTM with 20 hidden units is used as the actor network and a three-layer neural network (NN) with 10 hidden units is used as the critic network.

\subsubsection{Parameters Setting} 
For common hyper-parameters, 
the number of negative samples $K=31$ (i.e., a batch size $bs=32$), 
the initial learning rate is searched from $lr \in \{10^{-3}, 10^{-4}, 10^{-5}\}$, 
the temperature parameters $\tau=0.1$, 
and the patience to wait for convergence is $50$ epochs (i.e., the process is terminated when the loss no longer drops for 50 epochs). 
For LRL, the momentum term $m=10^{-3}$.


\subsection{Efficiency Analysis: How fast ICL is?}
We first perform the image and graph-level classification on CV and GRL datasets.
To simulate different incremental scenarios, we randomly split each dataset into the old data and new data according to the given growth ratio $\alpha$.
As an unbiased approach, the proposed ICL framework achieves high efficiency and fast convergence. 
Specifically, we have the following observations.

\textbf{High Efficiency.} 
The proposed ICL framework has a great superiority in terms of training time consumption as shown in Table~\ref{tab:classification_cv_time} and Table~\ref{tab:classification_grl_time}, where we use the speedup of training time compared with retraining $s_{e,i}=time_{retrain}/time_i$ for method $i$ as the metric. 
(a) For CV datasets, ICL achieves a speedup of up to $16.7\times$ w.r.t. retraining. 
Moreover, without the meta-optimization and LRL (i.e., only applying NCE-II), ICL still brings us a speedup of up to $9.8\times$.
Overall, it is obvious that ICL is the most efficient approach in all of the cases on CV datasets, even compared with fine-tuning and distillation which exclude the need of training old data. 
(b) For GRL datasets, ICL also achieves a speedup of up to $10.1\times$, and still $5.9\times$ without the meta-optimization and LRL. 
It is worth mentioning that ICL can beat replay and distillation with only NCE-II.

\textbf{Fast Convergence.} 
The proposed ICL framework gives the fastest convergence speed. We use the speedup of convergence epoch compared with the retraining $s_{c,i}=epoch_{retrain}/epoch_i$ for method $i$ as its metric. 
(a) For CV datasets, ICL achieves a speedup of up to $16.8\times$ compared with retraining and the improvements are significant in all cases. Moreover, ICL only with NCE-II (w/o M+LRL) still gains a speedup of up to $16.4\times$.
Similar to training time, the convergence speedup of ICL exceeds all methods.
(b) For GRL datasets, the experimental results are identical. ICL with its variants achieves a speedup of up to $14.6\times$. 

Therefore, the proposed ICL maintains significant advantages in reducing training time and accelerating convergence. Furthermore, although the time complexity of ICL is larger,
the results still show the superiority of ICL in terms of efficiency, 
which contributed to the LRL mechanism. 

\subsection{Effectiveness Analysis: Does ICL impede model?} 
As an unbiased and efficient approach, the proposed ICL also achieves competitive performance as shown in Figure~\ref{fig:acc_old} and Figure~\ref{fig:acc_new}. 
It is clearly observed that applying ICL to a contrastive learning approach will not impede the representation ability of the model. Specifically, the difference between ICL and the others falls in $[-0.0137, +0.0126]$ with an average improvement of 0.002.

\subsection{Incremental Setting Analysis: When to use ICL?}
We study the change of speedup of training time and convergence epoch with the variation of the growth ratio $\alpha$. For fairness, we compare the unbiased methods (ICL and retraining) and report the results in Figure~\ref{fig:alpha}. 
Specifically, we have the following findings.

\textbf{Consistent Efficiency.} We vary $\alpha$ from $0.1$ to $0.9$ at $0.1$ intervals, i.e., simulating the amount of new data from 1/9 of the old data to 9 times the old data. Overall, ICL achieves a $2.2\times$-$18.3\times$ speedup of training time and a $2.7\times$-$22.9\times$ speedup of convergence epoch.

\textbf{Adaption Ability with a Large Ratio of New Data.} The superiority of ICL in training time is obvious even when $\alpha$ is large. Moreover, from Table~\ref{tab:classification_cv_time} and Table~\ref{tab:classification_grl_time}, ICL consistently achieves faster convergence while the other baselines nearly equal the retraining. Thus, ICL is more practical in scenarios where there is a large ratio of new data.

\textbf{Limited Superiority by the Ratio of New Data.} As $\alpha$ increases, the speedup of ICL becomes smaller due to a large amount of new data. Thus, overmuch new data leads to the weakening of incremental learning strategies. However, ICL nevertheless gains a $2.2\times$-$5.3\times$ speedup when the amount of the new data is 9 times the old one.

\subsection{Ablation Study: How ICL works?}
We further investigate the ablation study of the three important components: the NCE-II loss, the meta-optimization algorithm, and the LRL mechanism. 
We provide the loss change compared with NCE-I and Adam in Figure~\ref{fig:training_loss}.

\textbf{NCE-II Loss is unbiased and useful.} From Table~\ref{tab:classification_cv_time} and Table~\ref{tab:classification_grl_time}, it is noticed that ICL only with NCE-II (i.e., without meta-optimization and LRL) still achieves a faster training and convergence speed, up to $9.8\times$ and $16.4\times$ respectively. From Figure~\ref{fig:training_loss}, NCE-II's decrease rate of loss is faster than only applying NCE-I, which reflects that NCE-II helps to quickly adapt to new data with an unbiased estimation.

\textbf{Meta-optimization is essential.} As shown in  Table~\ref{tab:classification_cv_time} and Table~\ref{tab:classification_grl_time}, without meta-optimization, the efficiency is significantly lowered. Moreover, comparing ICL only with NCE-II and ICL with meta-optimization, the efficiency is improved. Thus, meta-optimization contributes considerably to fast adaption of new data.

\textbf{Learning Rate Learning has a vital contribution to the improvement of efficiency.} From Table~\ref{tab:classification_cv_time} and Table~\ref{tab:classification_grl_time}, without LRL mechanism, the degree of improvement on efficiency is reduced. In Figure~\ref{fig:training_loss}, it is obvious that the loss declines faster with LRL compared with the Adam algorithm.

\begin{figure*}[t]
	\centering
	\subfigure[ImageNet-2.]{
		\includegraphics[width=0.22\textwidth]{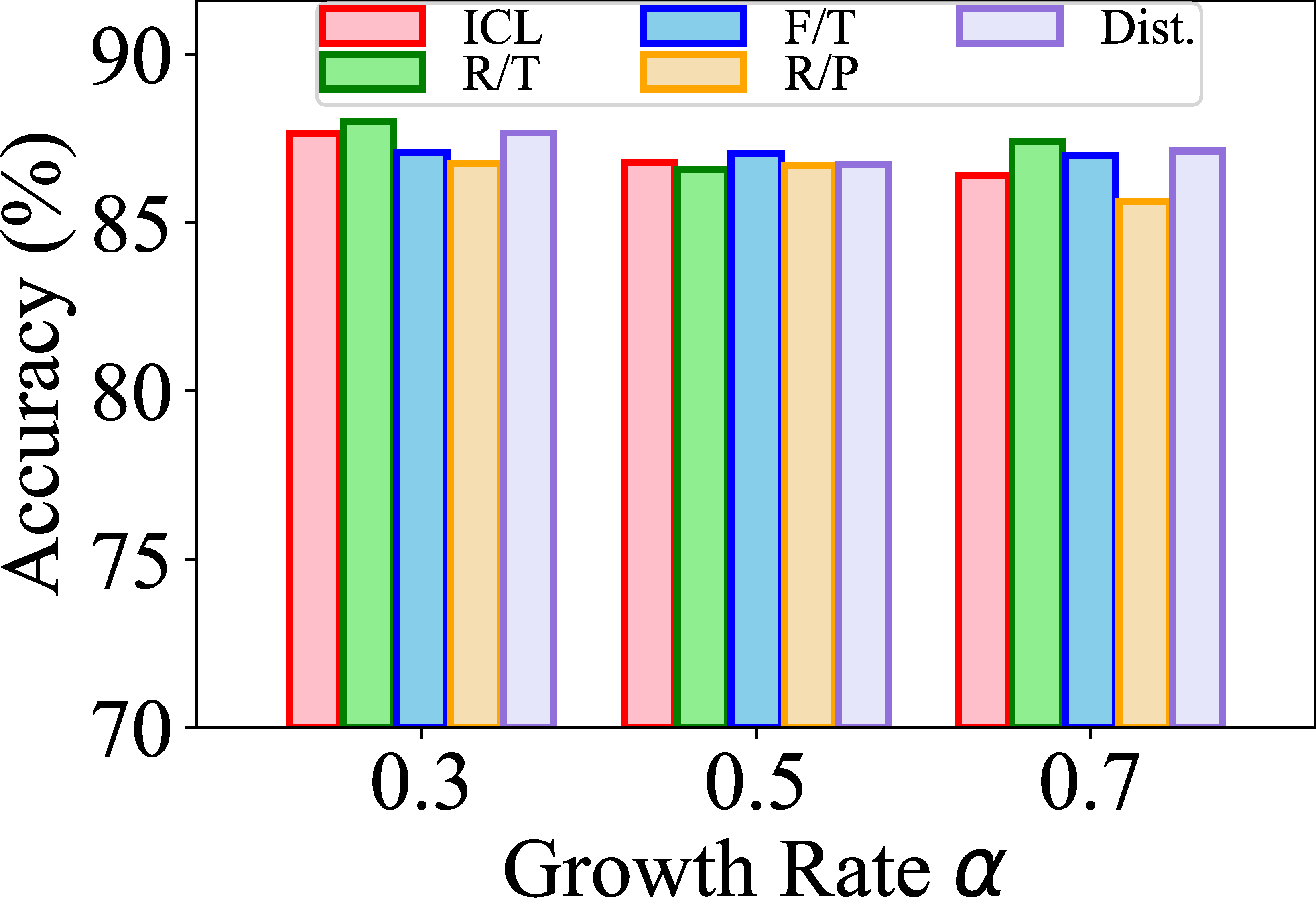}
	}
	\subfigure[MNIST-2.]{
		\includegraphics[width=0.225\textwidth]{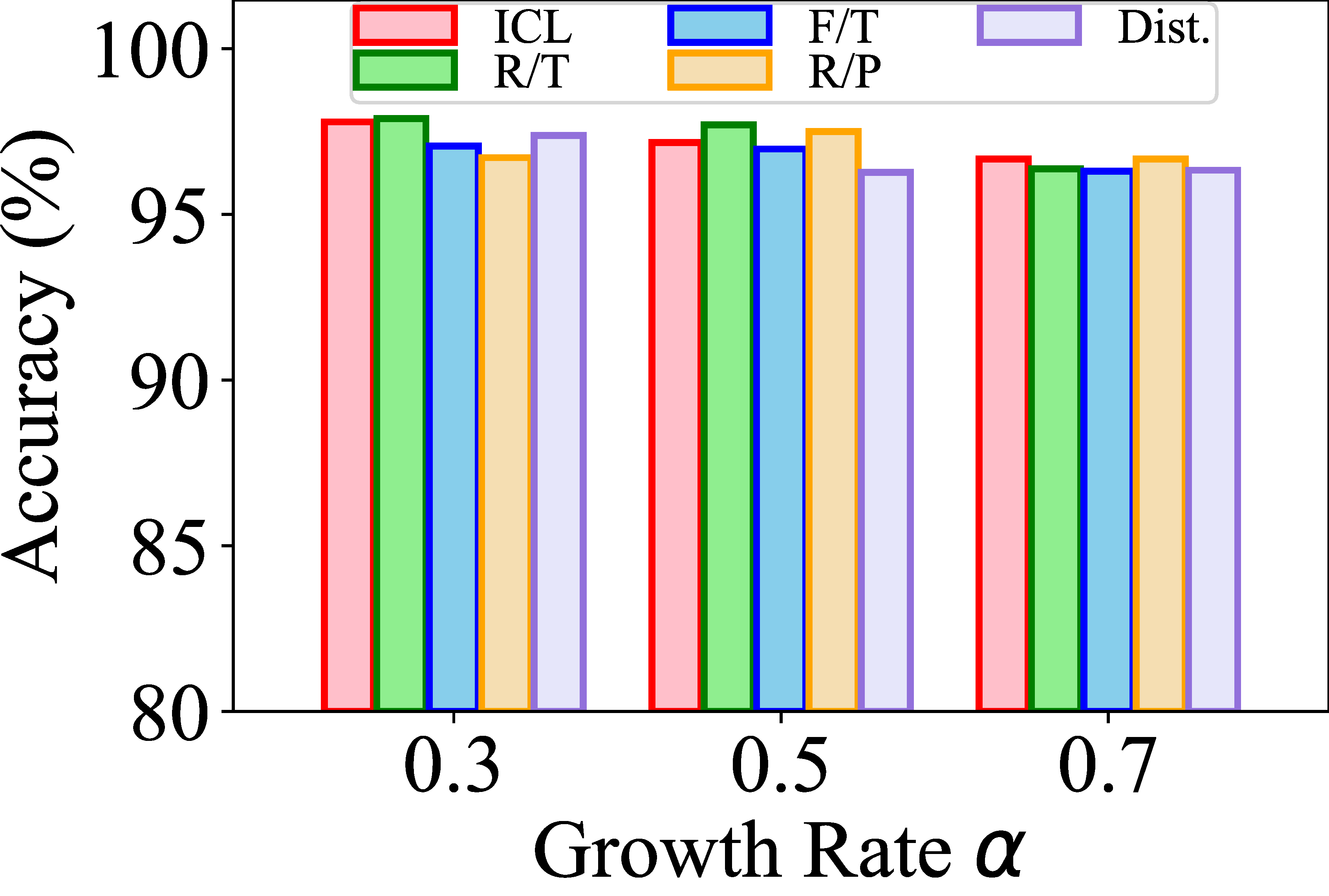}
	}
	\subfigure[PROTEINS.]{
		\includegraphics[width=0.22\textwidth]{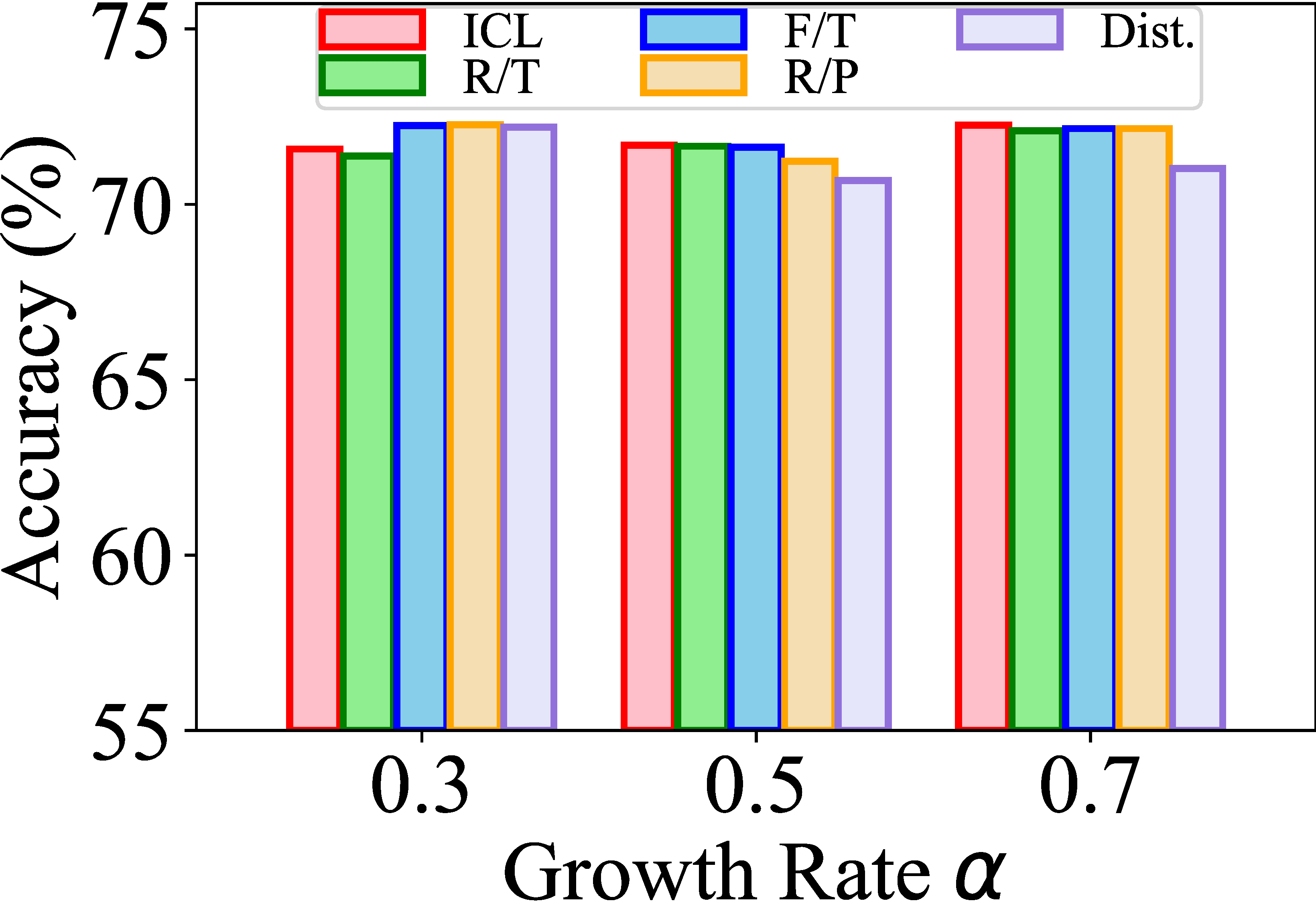}
	}
	\subfigure[REDDIT.]{
		\includegraphics[width=0.22\textwidth]{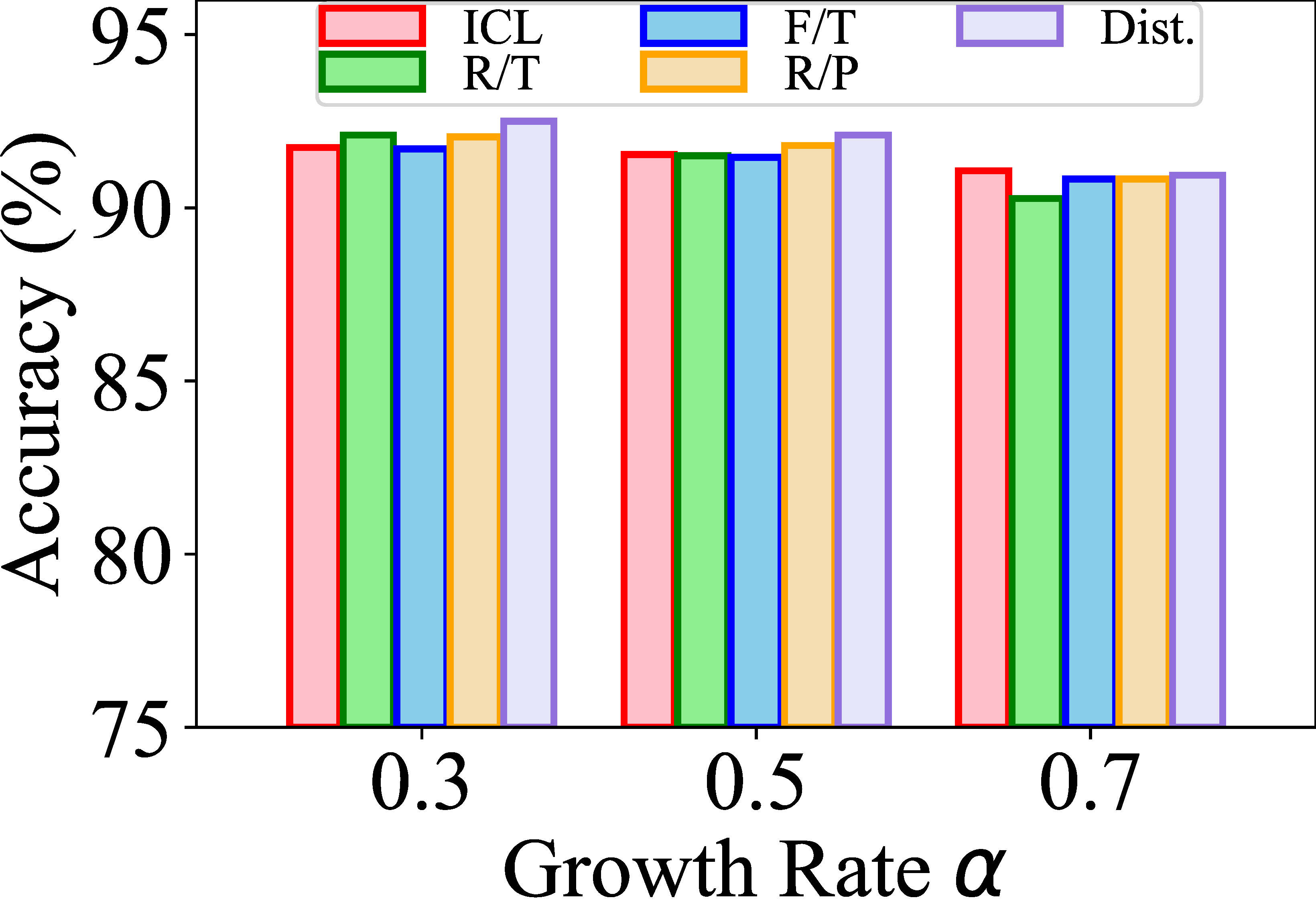}
	}
	\caption{Classification results of old data on four datasets. (R/T: Retraining; F/T: Fine-tuning; R/P: Replay; Dist.: Distillation)}\label{fig:acc_old}
\end{figure*}

\begin{figure*}[t]
	\centering
	\subfigure[ImageNet-2.]{
		\includegraphics[width=0.22\textwidth]{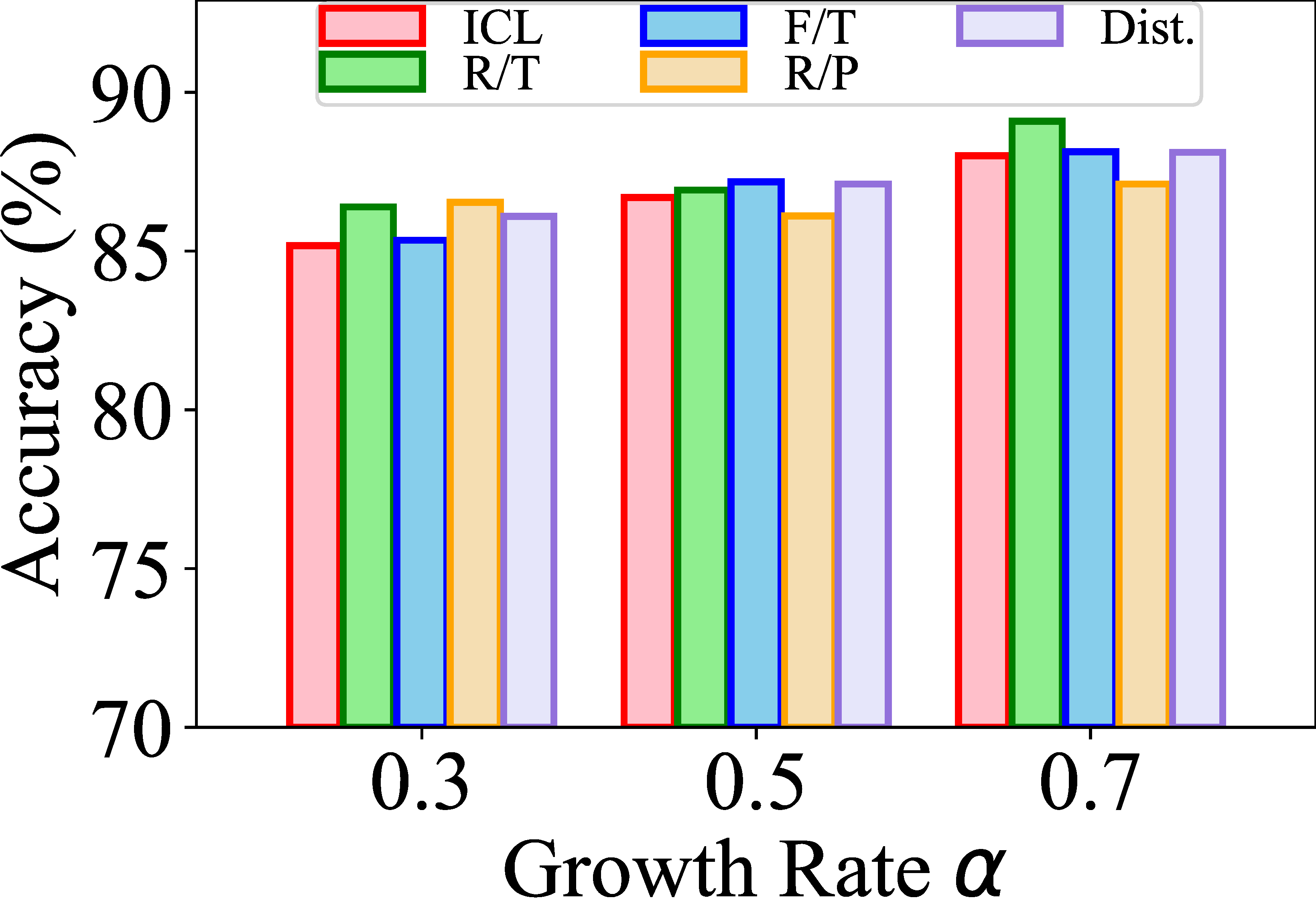}
	}
	\subfigure[MNIST-2.]{
		\includegraphics[width=0.225\textwidth]{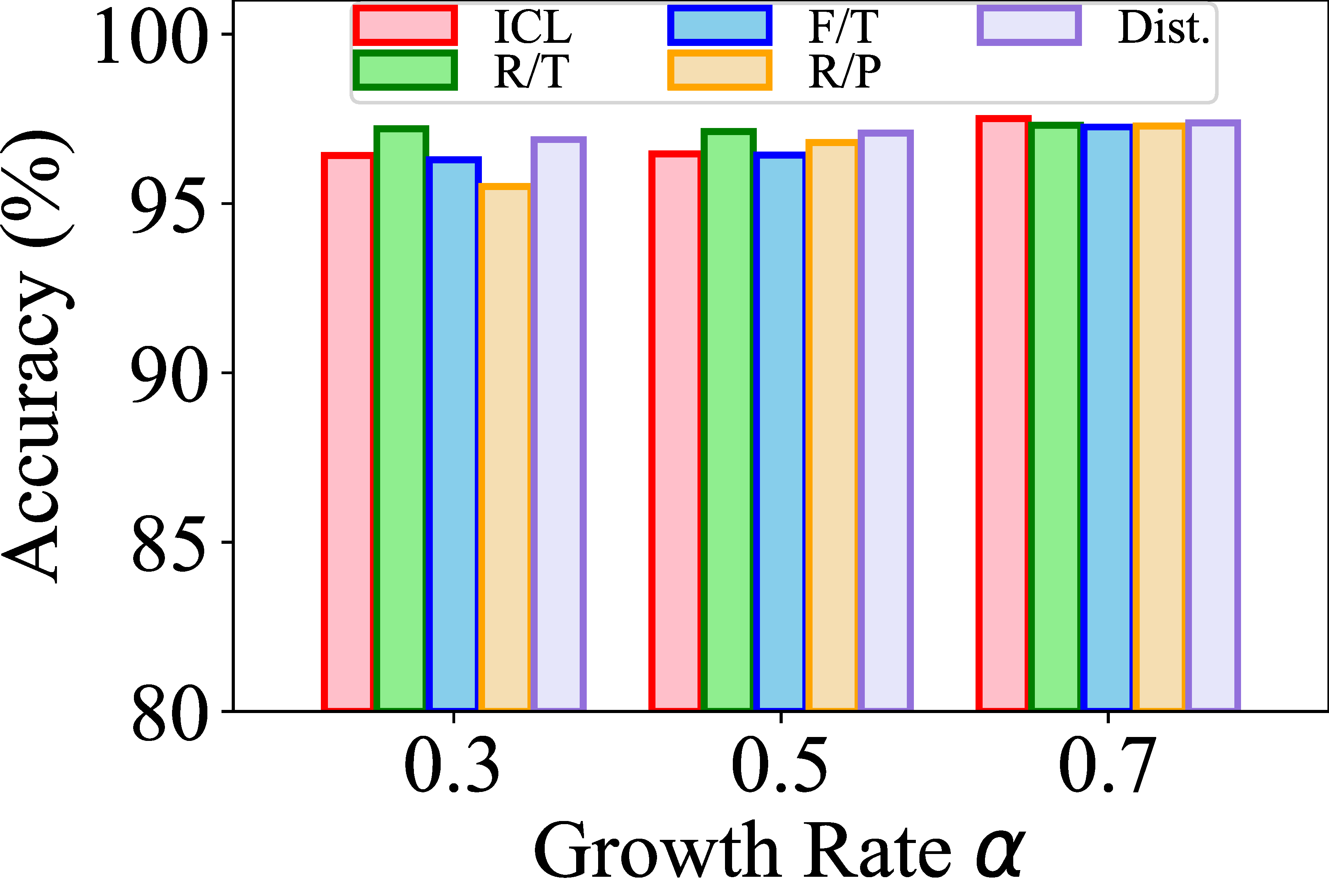}
	}
	\subfigure[PROTEINS.]{
		\includegraphics[width=0.22\textwidth]{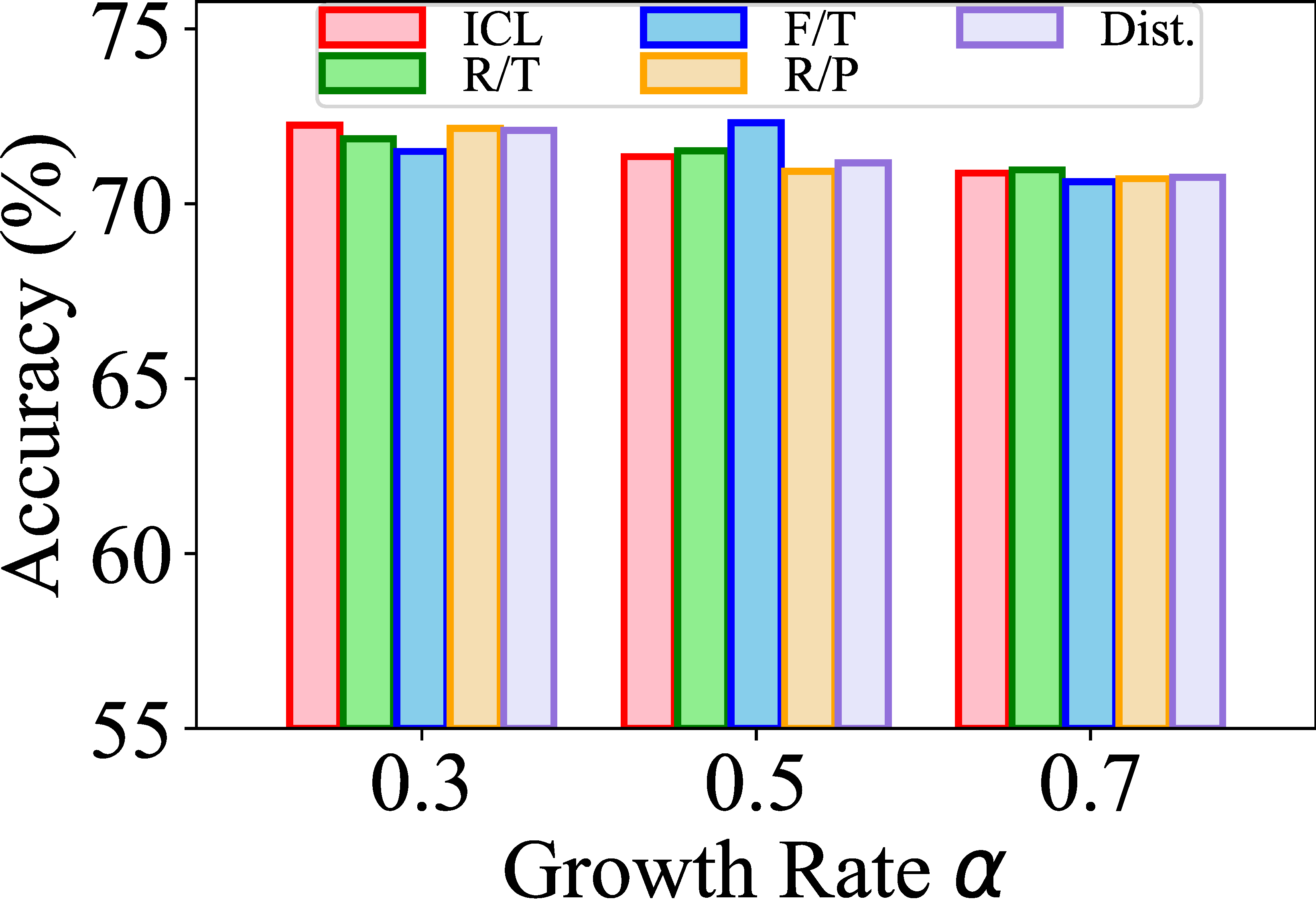}
	}
	\subfigure[REDDIT.]{
		\includegraphics[width=0.22\textwidth]{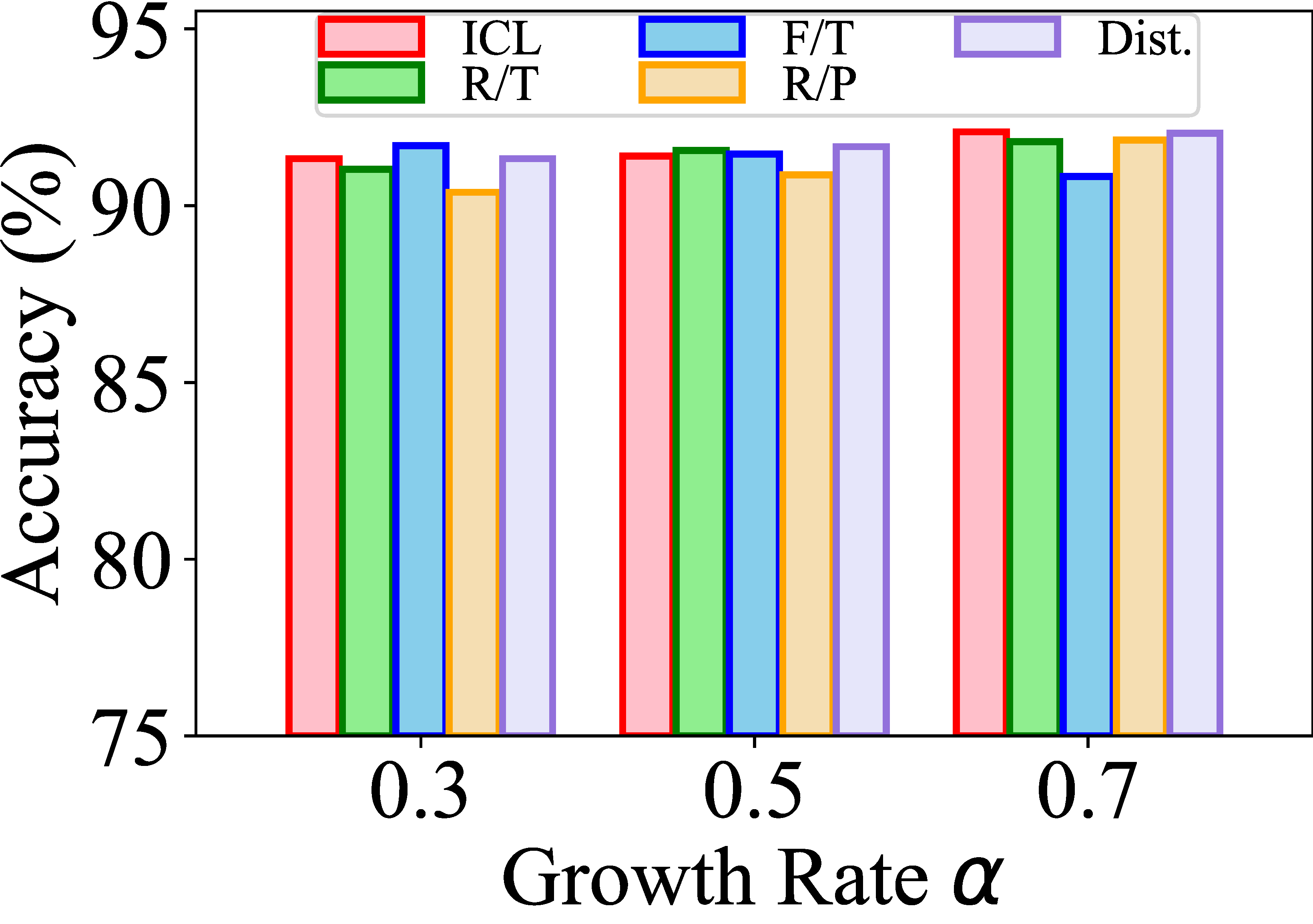}
	}
	\caption{Classification results of new data on four datasets. (R/T: Retraining; F/T: Fine-tuning; R/P: Replay; Dist.: Distillation)}\label{fig:acc_new}
\end{figure*}

\begin{figure}[t]
	\centering
	\subfigure[ImageNet-2 (Time).]{
		\includegraphics[width=0.22\textwidth]{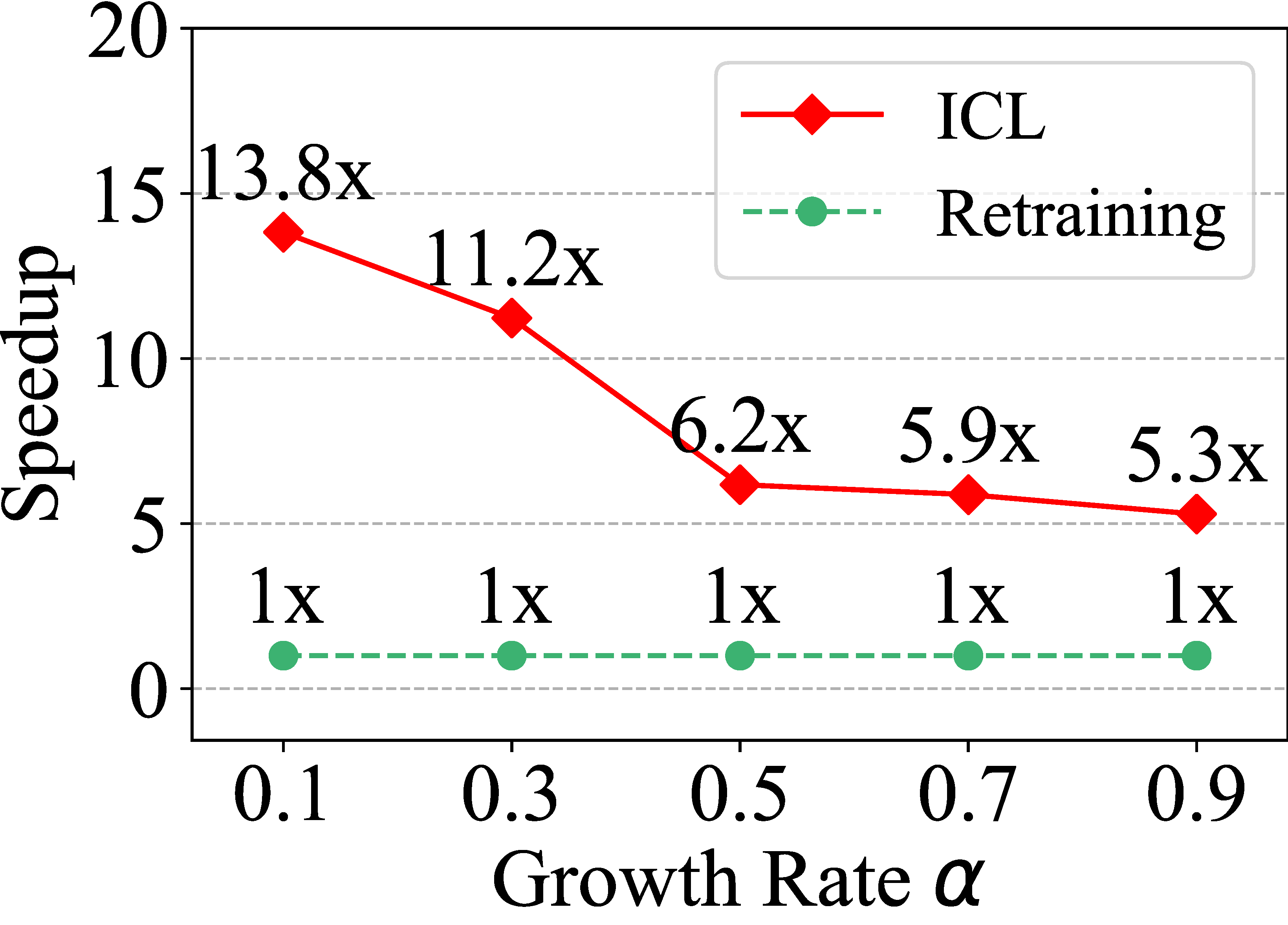}
	}
	\subfigure[ImageNet-2 (Epoch).]{
		\includegraphics[width=0.22\textwidth]{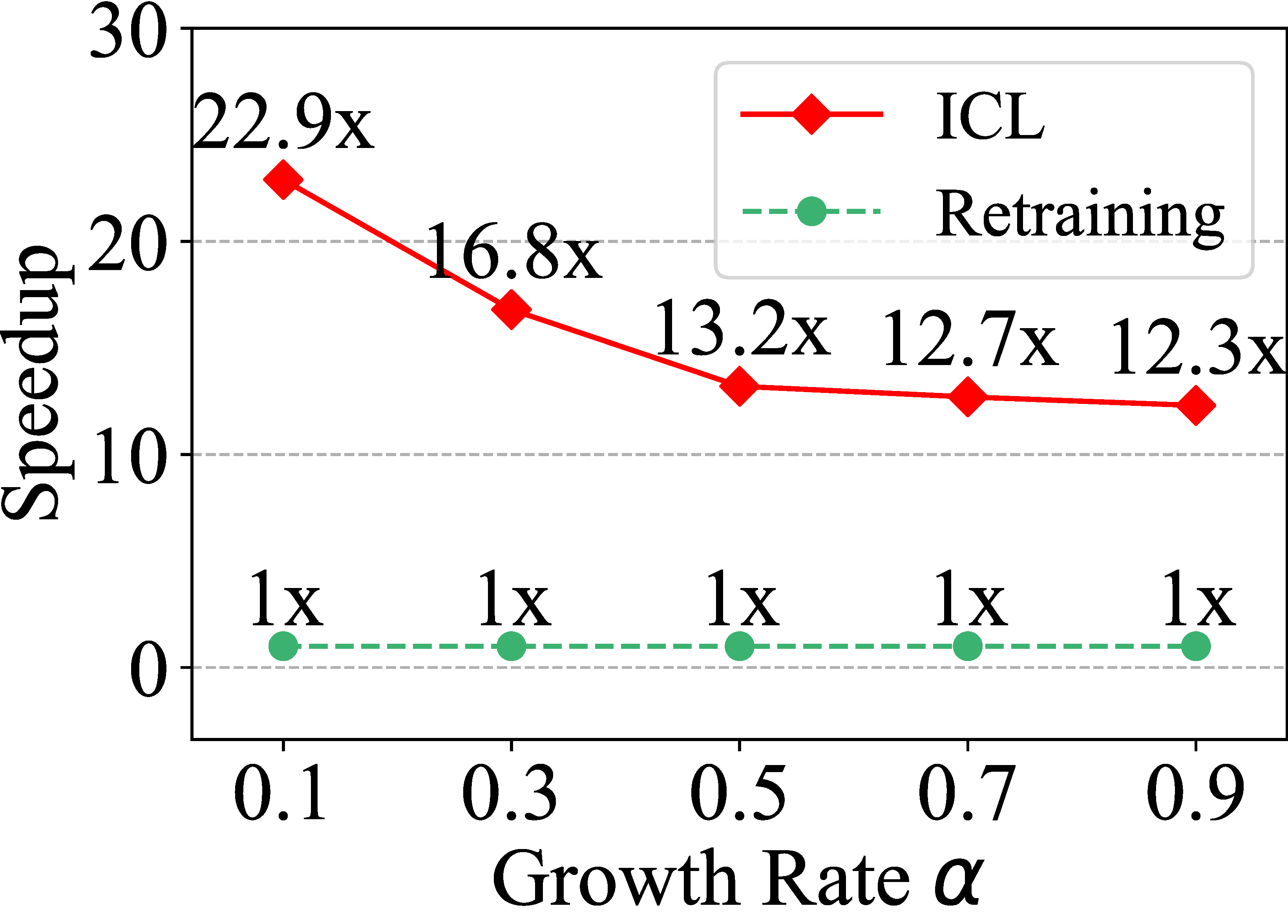}
	}
	\subfigure[PROTEINS (Time).]{
		\includegraphics[width=0.22\textwidth]{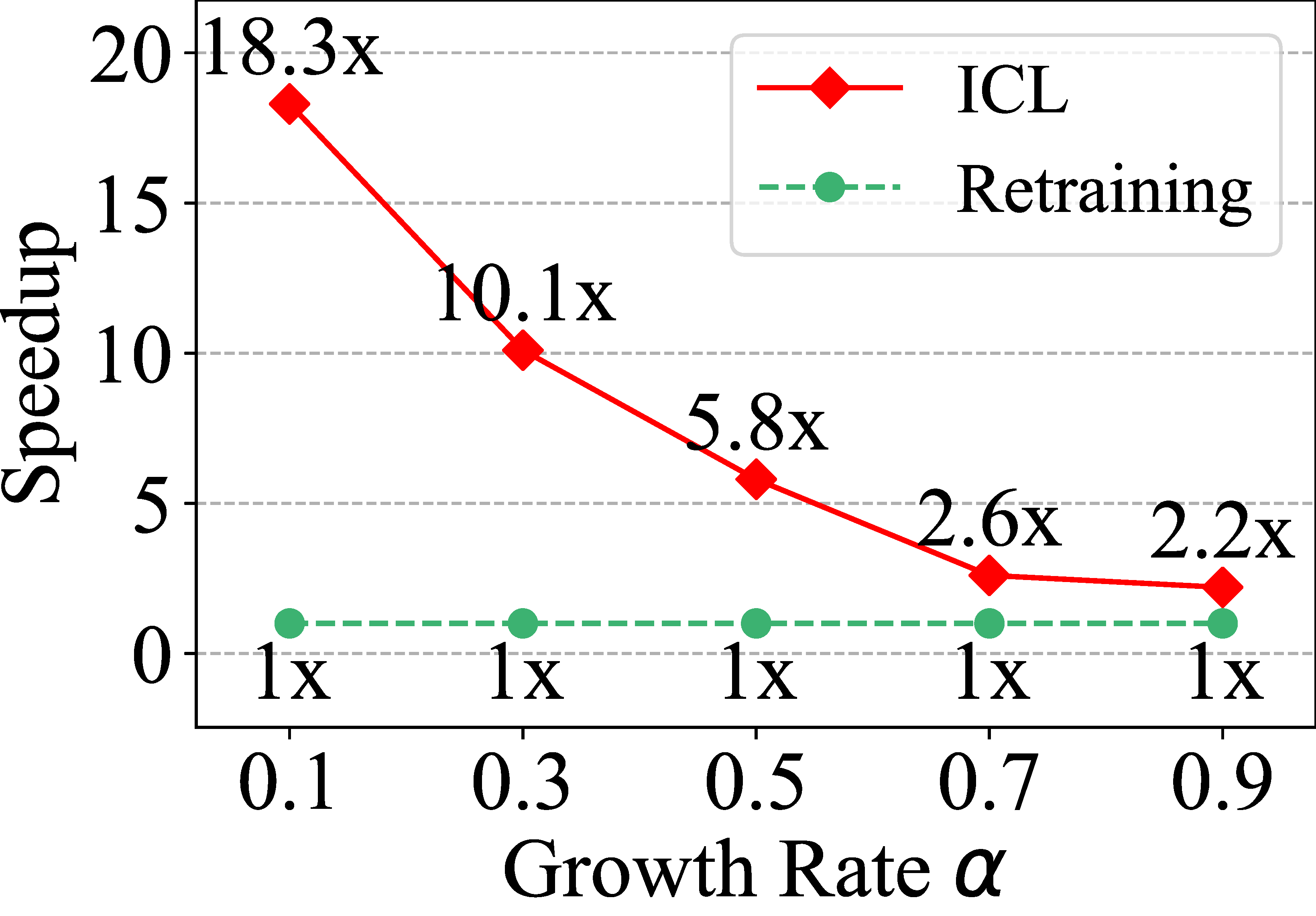}
	}
	\subfigure[PROTEINS (Epoch).]{
		\includegraphics[width=0.22\textwidth]{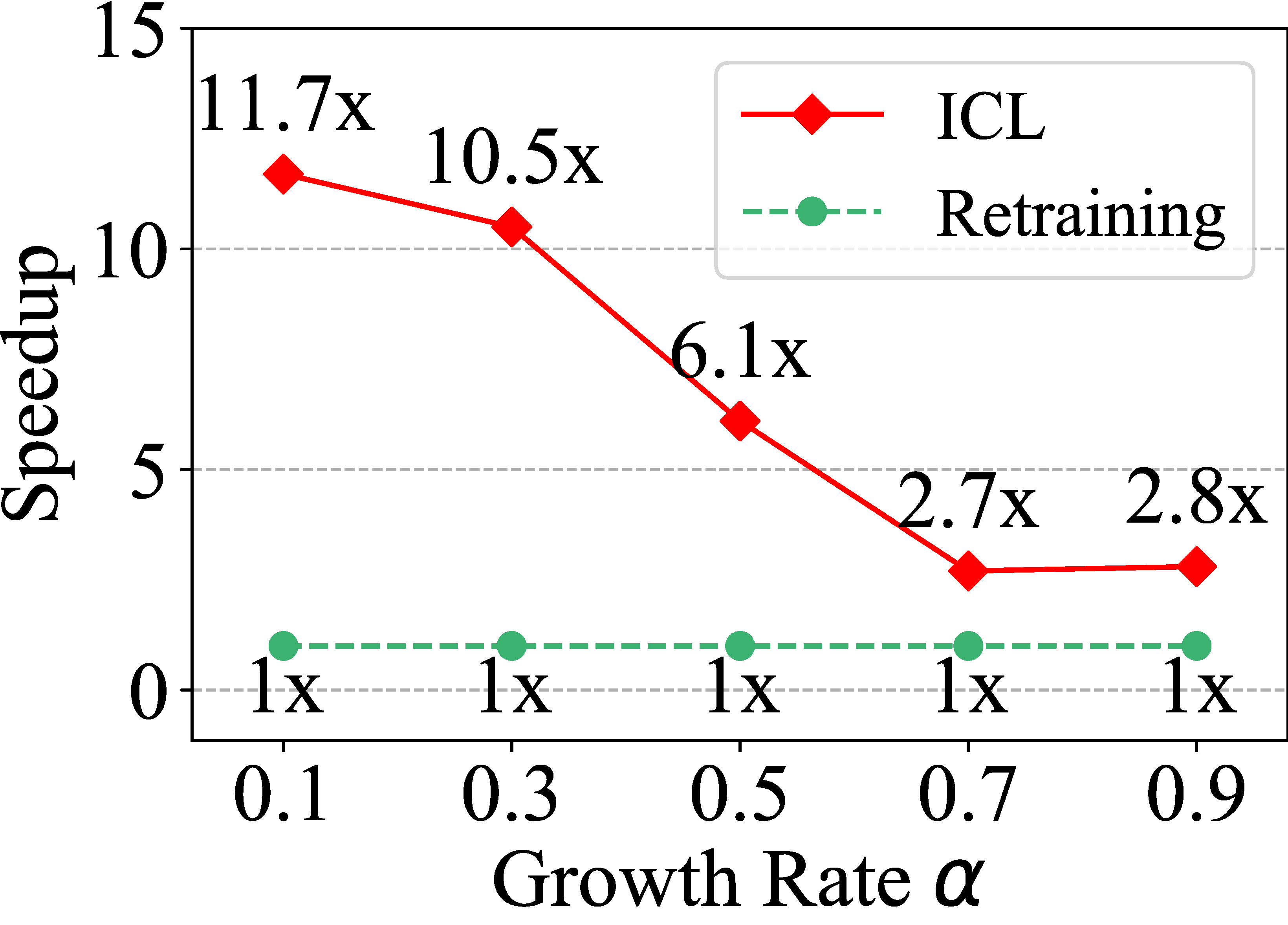}
	}
	\caption{Speedup of unbiased methods (ICL and retraining) with the variation of $\alpha$. For fairness, the biased ones (fine-tuning, replay, and distillation) are excluded.}
	\label{fig:alpha}
\end{figure}

\begin{figure}[t]
	\centering
	\subfigure[ImageNet-2.]{
		\includegraphics[width=0.22\textwidth]{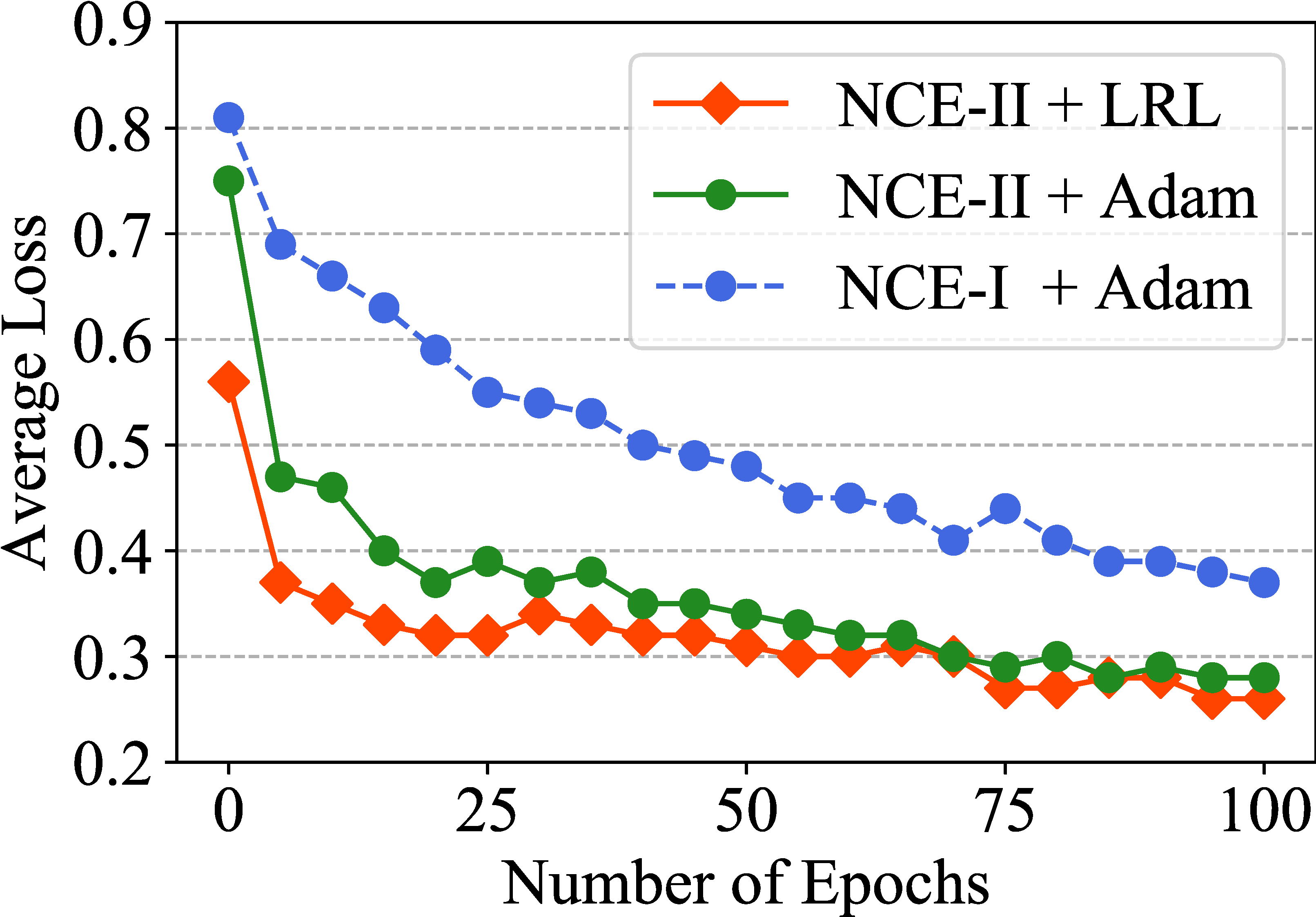}
	}
	\subfigure[REDDIT.]{
		\includegraphics[width=0.22\textwidth]{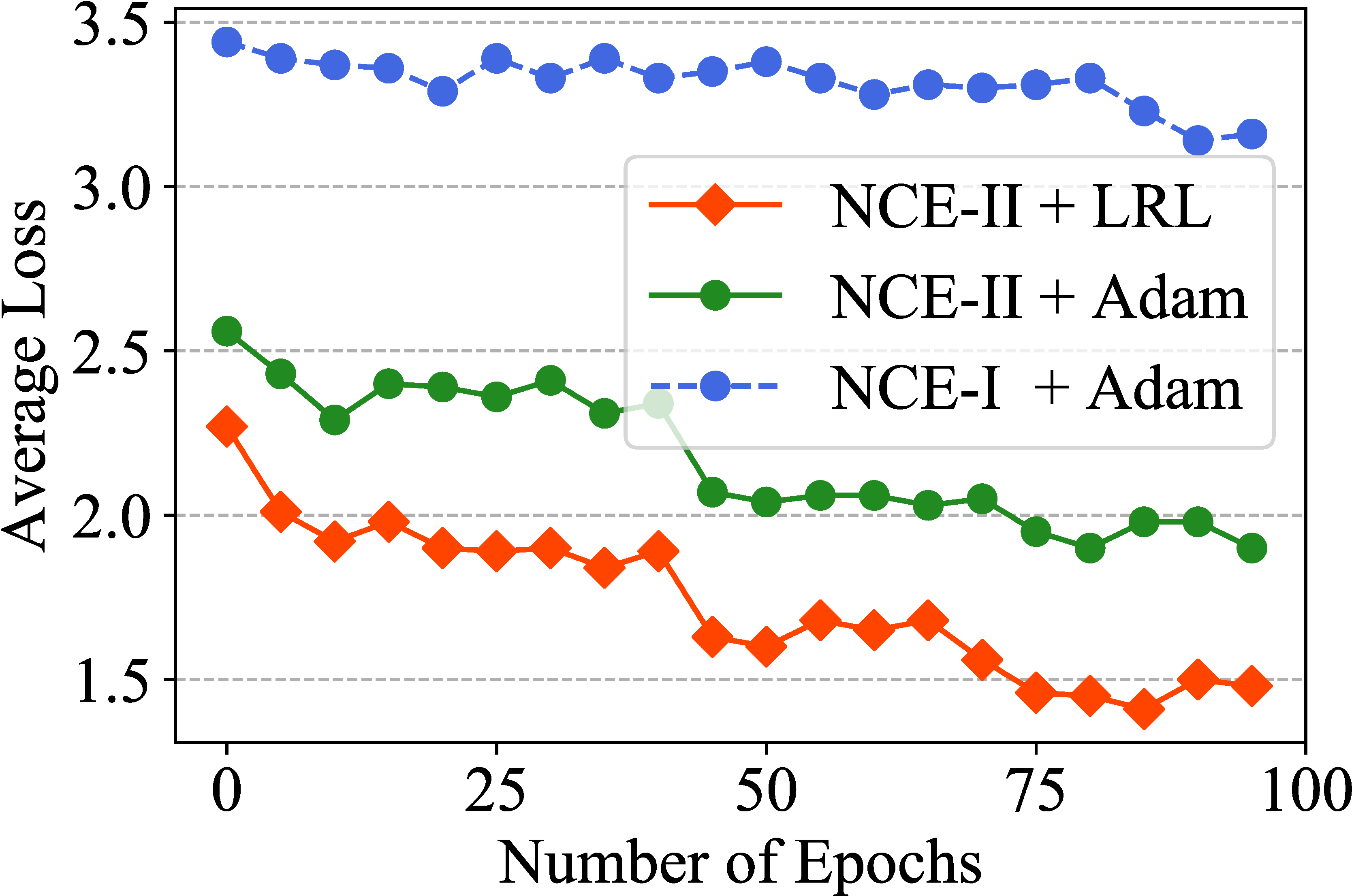}
	}
	\caption{Training process on ImageNet-2 and REDDIT.}
	\label{fig:training_loss}
\end{figure}

\section{Related Work}\label{sec:related_work}

\subsection{Contrastive Learning}
Contrastive learning (CL) is a self-supervised approach that aims at learning to discriminate the samples by contrasting the negative ones~\cite{liu2021self,jaiswal2021survey,li2021rwne}.
Contrastive Predictive Coding (CPC) proposed an InfoNCE loss to maximize the mutual information between the sample and its positive one~\cite{oord2018representation, sun2022graph}. Recently, InfoNCE loss has been widely used in computer vision (CV). SimCLR and MoCo subsequently use contrastive learning to generate the representations of images through different data augmentation methods~\cite{chen2020simple,he2020momentum} and achieve comparable results with state-of-the-art supervised approaches. Furthermore, contrastive learning is also well researched in graph representation learning (GRL). Graph Contrastive Coding (GCC) and GrpahCL apply InfoNCE to learn the representation of node and graph~\cite{qiu2020gcc,you2020graph}. The above studies follow a similar framework to Definition~\ref{def:contrastive_learning} which is the backbone of ICL.  

\subsection{Incremental Learning}
Incremental learning is a learning process where the new data is continuously coming from the environment~\cite{ade2013methods,peng2017incrementally,peng2021lime}. 
Most studies of incremental learning focus on supervised learning. For example, iCaRL proposed an incremental classifier and representation learning approach for supervised incremental learning~\cite{rebuffi2017icarl}. In addition, \cite{castro2018end} proposed an end-to-end incremental learning method for class-incremental issue. However, these methods are hard to implement into self-supervised contrastive learning due to the lack of labels. Recently, some works incorporate the replay and distillation technique into CL~\cite{lin2021continual,cha2021co}, while they still have the bias issue.

\section{Conclusion}\label{sec:conclusion}
In this paper, we studied the unbiased incremental learning issue in self-supervised contrastive learning and proposed \emph{Incremental Contrastive Learning} (ICL) framework. Specifically, we designed an Incremental InfoNCE (NCE-II) loss function to give an unbiased estimation of noise distribution change in incremental scenarios. Moreover, we proposed a meta-optimization algorithm with Learning Rate Learning (LRL) mechanism to achieve fast convergence. The experiments demonstrated the efficiency and effectiveness of the proposed ICL framework.

\begin{acks}
The corresponding author is Jianxin Li. 
This work is supported in part by NSFC through grant U20B2053 and NSF under grants III-1763325, III-1909323, III-2106758, and SaTC-1930941.
\end{acks}

\clearpage
\bibliographystyle{ACM-Reference-Format}
\balance
\bibliography{ICL}


\begin{thebibliography}{44}


\ifx \showCODEN    \undefined \def \showCODEN     #1{\unskip}     \fi
\ifx \showDOI      \undefined \def \showDOI       #1{#1}\fi
\ifx \showISBNx    \undefined \def \showISBNx     #1{\unskip}     \fi
\ifx \showISBNxiii \undefined \def \showISBNxiii  #1{\unskip}     \fi
\ifx \showISSN     \undefined \def \showISSN      #1{\unskip}     \fi
\ifx \showLCCN     \undefined \def \showLCCN      #1{\unskip}     \fi
\ifx \shownote     \undefined \def \shownote      #1{#1}          \fi
\ifx \showarticletitle \undefined \def \showarticletitle #1{#1}   \fi
\ifx \showURL      \undefined \def \showURL       {\relax}        \fi
\providecommand\bibfield[2]{#2}
\providecommand\bibinfo[2]{#2}
\providecommand\natexlab[1]{#1}
\providecommand\showeprint[2][]{arXiv:#2}

\bibitem[Ade and Deshmukh(2013)]%
        {ade2013methods}
\bibfield{author}{\bibinfo{person}{RR Ade} {and} \bibinfo{person}{PR
  Deshmukh}.} \bibinfo{year}{2013}\natexlab{}.
\newblock \showarticletitle{Methods for incremental learning: a survey}.
\newblock \bibinfo{journal}{\emph{International Journal of Data Mining \&
  Knowledge Management Process}} \bibinfo{volume}{3}, \bibinfo{number}{4}
  (\bibinfo{year}{2013}), \bibinfo{pages}{119}.
\newblock


\bibitem[Cai and Wang(2018)]%
        {cai2018simple}
\bibfield{author}{\bibinfo{person}{Chen Cai} {and} \bibinfo{person}{Yusu
  Wang}.} \bibinfo{year}{2018}\natexlab{}.
\newblock \showarticletitle{A simple yet effective baseline for non-attributed
  graph classification}.
\newblock \bibinfo{journal}{\emph{arXiv preprint arXiv:1811.03508}}
  (\bibinfo{year}{2018}).
\newblock


\bibitem[Castro et~al\mbox{.}(2018)]%
        {castro2018end}
\bibfield{author}{\bibinfo{person}{Francisco~M Castro},
  \bibinfo{person}{Manuel~J Mar{\'\i}n-Jim{\'e}nez},
  \bibinfo{person}{Nicol{\'a}s Guil}, \bibinfo{person}{Cordelia Schmid}, {and}
  \bibinfo{person}{Karteek Alahari}.} \bibinfo{year}{2018}\natexlab{}.
\newblock \showarticletitle{End-to-end incremental learning}. In
  \bibinfo{booktitle}{\emph{Proceedings of the European conference on computer
  vision (ECCV)}}. \bibinfo{pages}{233--248}.
\newblock


\bibitem[Cha et~al\mbox{.}(2021)]%
        {cha2021co}
\bibfield{author}{\bibinfo{person}{Hyuntak Cha}, \bibinfo{person}{Jaeho Lee},
  {and} \bibinfo{person}{Jinwoo Shin}.} \bibinfo{year}{2021}\natexlab{}.
\newblock \showarticletitle{Co2l: Contrastive continual learning}. In
  \bibinfo{booktitle}{\emph{Proceedings of the IEEE/CVF International
  Conference on Computer Vision (ICCV)}}. \bibinfo{pages}{9516--9525}.
\newblock


\bibitem[Chen et~al\mbox{.}(2020)]%
        {chen2020simple}
\bibfield{author}{\bibinfo{person}{Ting Chen}, \bibinfo{person}{Simon
  Kornblith}, \bibinfo{person}{Mohammad Norouzi}, {and}
  \bibinfo{person}{Geoffrey Hinton}.} \bibinfo{year}{2020}\natexlab{}.
\newblock \showarticletitle{A simple framework for contrastive learning of
  visual representations}. In \bibinfo{booktitle}{\emph{International
  conference on machine learning (ICML)}}. PMLR, \bibinfo{pages}{1597--1607}.
\newblock


\bibitem[Deng et~al\mbox{.}(2009)]%
        {deng2009imagenet}
\bibfield{author}{\bibinfo{person}{Jia Deng}, \bibinfo{person}{Wei Dong},
  \bibinfo{person}{Richard Socher}, \bibinfo{person}{Li-Jia Li},
  \bibinfo{person}{Kai Li}, {and} \bibinfo{person}{Li Fei-Fei}.}
  \bibinfo{year}{2009}\natexlab{}.
\newblock \showarticletitle{Imagenet: A large-scale hierarchical image
  database}. In \bibinfo{booktitle}{\emph{2009 IEEE conference on computer
  vision and pattern recognition (CVPR)}}. Ieee, \bibinfo{pages}{248--255}.
\newblock


\bibitem[Fang et~al\mbox{.}(2020)]%
        {fang2020seed}
\bibfield{author}{\bibinfo{person}{Zhiyuan Fang}, \bibinfo{person}{Jianfeng
  Wang}, \bibinfo{person}{Lijuan Wang}, \bibinfo{person}{Lei Zhang},
  \bibinfo{person}{Yezhou Yang}, {and} \bibinfo{person}{Zicheng Liu}.}
  \bibinfo{year}{2020}\natexlab{}.
\newblock \showarticletitle{SEED: Self-supervised Distillation For Visual
  Representation}. In \bibinfo{booktitle}{\emph{International Conference on
  Learning Representations (ICLR)}}.
\newblock


\bibitem[Finn et~al\mbox{.}(2017)]%
        {finn2017model}
\bibfield{author}{\bibinfo{person}{Chelsea Finn}, \bibinfo{person}{Pieter
  Abbeel}, {and} \bibinfo{person}{Sergey Levine}.}
  \bibinfo{year}{2017}\natexlab{}.
\newblock \showarticletitle{Model-agnostic meta-learning for fast adaptation of
  deep networks}. In \bibinfo{booktitle}{\emph{International conference on
  machine learning (ICML)}}. PMLR, \bibinfo{pages}{1126--1135}.
\newblock


\bibitem[Gutmann and Hyv{\"a}rinen(2010)]%
        {gutmann2010noise}
\bibfield{author}{\bibinfo{person}{Michael Gutmann} {and} \bibinfo{person}{Aapo
  Hyv{\"a}rinen}.} \bibinfo{year}{2010}\natexlab{}.
\newblock \showarticletitle{Noise-contrastive estimation: A new estimation
  principle for unnormalized statistical models}. In
  \bibinfo{booktitle}{\emph{Proceedings of the thirteenth international
  conference on artificial intelligence and statistics}}. JMLR Workshop and
  Conference Proceedings, \bibinfo{pages}{297--304}.
\newblock


\bibitem[Gutmann and Hyv{\"a}rinen(2012)]%
        {gutmann2012noise}
\bibfield{author}{\bibinfo{person}{Michael~U Gutmann} {and}
  \bibinfo{person}{Aapo Hyv{\"a}rinen}.} \bibinfo{year}{2012}\natexlab{}.
\newblock \showarticletitle{Noise-Contrastive Estimation of Unnormalized
  Statistical Models, with Applications to Natural Image Statistics.}
\newblock \bibinfo{journal}{\emph{Journal of machine learning research}}
  \bibinfo{volume}{13}, \bibinfo{number}{2} (\bibinfo{year}{2012}).
\newblock


\bibitem[He et~al\mbox{.}(2020)]%
        {he2020momentum}
\bibfield{author}{\bibinfo{person}{Kaiming He}, \bibinfo{person}{Haoqi Fan},
  \bibinfo{person}{Yuxin Wu}, \bibinfo{person}{Saining Xie}, {and}
  \bibinfo{person}{Ross Girshick}.} \bibinfo{year}{2020}\natexlab{}.
\newblock \showarticletitle{Momentum contrast for unsupervised visual
  representation learning}. In \bibinfo{booktitle}{\emph{Proceedings of the
  IEEE/CVF conference on computer vision and pattern recognition (CVPR)}}.
  \bibinfo{pages}{9729--9738}.
\newblock


\bibitem[He et~al\mbox{.}(2016)]%
        {he2016deep}
\bibfield{author}{\bibinfo{person}{Kaiming He}, \bibinfo{person}{Xiangyu
  Zhang}, \bibinfo{person}{Shaoqing Ren}, {and} \bibinfo{person}{Jian Sun}.}
  \bibinfo{year}{2016}\natexlab{}.
\newblock \showarticletitle{Deep residual learning for image recognition}. In
  \bibinfo{booktitle}{\emph{Proceedings of the IEEE conference on computer
  vision and pattern recognition (CVPR)}}. \bibinfo{pages}{770--778}.
\newblock


\bibitem[Jaiswal et~al\mbox{.}(2020)]%
        {jaiswal2021survey}
\bibfield{author}{\bibinfo{person}{Ashish Jaiswal},
  \bibinfo{person}{Ashwin~Ramesh Babu}, \bibinfo{person}{Mohammad~Zaki Zadeh},
  \bibinfo{person}{Debapriya Banerjee}, {and} \bibinfo{person}{Fillia
  Makedon}.} \bibinfo{year}{2020}\natexlab{}.
\newblock \showarticletitle{A survey on contrastive self-supervised learning}.
\newblock \bibinfo{journal}{\emph{Technologies}} \bibinfo{volume}{9},
  \bibinfo{number}{1} (\bibinfo{year}{2020}), \bibinfo{pages}{2}.
\newblock


\bibitem[Kaelbling et~al\mbox{.}(1996)]%
        {kaelbling1996reinforcement}
\bibfield{author}{\bibinfo{person}{Leslie~Pack Kaelbling},
  \bibinfo{person}{Michael~L Littman}, {and} \bibinfo{person}{Andrew~W Moore}.}
  \bibinfo{year}{1996}\natexlab{}.
\newblock \showarticletitle{Reinforcement learning: A survey}.
\newblock \bibinfo{journal}{\emph{Journal of artificial intelligence research}}
   \bibinfo{volume}{4} (\bibinfo{year}{1996}), \bibinfo{pages}{237--285}.
\newblock


\bibitem[Kipf and Welling(2017)]%
        {kipf2016semi}
\bibfield{author}{\bibinfo{person}{Thomas~N Kipf} {and} \bibinfo{person}{Max
  Welling}.} \bibinfo{year}{2017}\natexlab{}.
\newblock \showarticletitle{Semi-supervised classification with graph
  convolutional networks}. In \bibinfo{booktitle}{\emph{International
  Conference on Learning Representations (ICLR)}}.
\newblock


\bibitem[LeCun et~al\mbox{.}(1998)]%
        {lecun1998gradient}
\bibfield{author}{\bibinfo{person}{Yann LeCun}, \bibinfo{person}{L{\'e}on
  Bottou}, \bibinfo{person}{Yoshua Bengio}, {and} \bibinfo{person}{Patrick
  Haffner}.} \bibinfo{year}{1998}\natexlab{}.
\newblock \showarticletitle{Gradient-based learning applied to document
  recognition}.
\newblock \bibinfo{journal}{\emph{Proc. IEEE}} \bibinfo{volume}{86},
  \bibinfo{number}{11} (\bibinfo{year}{1998}), \bibinfo{pages}{2278--2324}.
\newblock


\bibitem[Li et~al\mbox{.}(2021a)]%
        {li2021rwne}
\bibfield{author}{\bibinfo{person}{Jianxin Li}, \bibinfo{person}{Cheng Ji},
  \bibinfo{person}{Hao Peng}, \bibinfo{person}{Yu He}, \bibinfo{person}{Yangqiu
  Song}, \bibinfo{person}{Xinmiao Zhang}, {and} \bibinfo{person}{Fanzhang
  Peng}.} \bibinfo{year}{2021}\natexlab{a}.
\newblock \showarticletitle{RWNE: A Scalable Random-Walk based Network
  Embedding Framework with Personalized Higher-order Proximity Preserved}.
\newblock \bibinfo{journal}{\emph{Journal of Artificial Intelligence Research}}
   \bibinfo{volume}{71} (\bibinfo{year}{2021}), \bibinfo{pages}{237--263}.
\newblock


\bibitem[Li et~al\mbox{.}(2022b)]%
        {li2022aiqoser}
\bibfield{author}{\bibinfo{person}{Jianxin Li}, \bibinfo{person}{Tianchen Zhu},
  \bibinfo{person}{Haoyi Zhou}, \bibinfo{person}{Qingyun Sun},
  \bibinfo{person}{Chunyang Jiang}, \bibinfo{person}{Shuai Zhang}, {and}
  \bibinfo{person}{Chunming Hu}.} \bibinfo{year}{2022}\natexlab{b}.
\newblock \showarticletitle{AIQoSer: Building the efficient Inference-QoS for
  AI Services}. In \bibinfo{booktitle}{\emph{2022 IEEE/ACM 30th International
  Symposium on Quality of Service (IWQoS)}}. IEEE, \bibinfo{pages}{1--10}.
\newblock


\bibitem[Li et~al\mbox{.}(2022a)]%
        {li2022survey}
\bibfield{author}{\bibinfo{person}{Qian Li}, \bibinfo{person}{Jianxin Li},
  \bibinfo{person}{Jiawei Sheng}, \bibinfo{person}{Shiyao Cui},
  \bibinfo{person}{Jia Wu}, \bibinfo{person}{Yiming Hei}, \bibinfo{person}{Hao
  Peng}, \bibinfo{person}{Shu Guo}, \bibinfo{person}{Lihong Wang},
  \bibinfo{person}{Amin Beheshti}, {et~al\mbox{.}}}
  \bibinfo{year}{2022}\natexlab{a}.
\newblock \showarticletitle{A Survey on Deep Learning Event Extraction:
  Approaches and Applications}.
\newblock \bibinfo{journal}{\emph{IEEE Transactions on Neural Networks and
  Learning Systems}} (\bibinfo{year}{2022}).
\newblock


\bibitem[Li et~al\mbox{.}(2021b)]%
        {li2021reinforcement}
\bibfield{author}{\bibinfo{person}{Qian Li}, \bibinfo{person}{Hao Peng},
  \bibinfo{person}{Jianxin Li}, \bibinfo{person}{Jia Wu},
  \bibinfo{person}{Yuanxing Ning}, \bibinfo{person}{Lihong Wang},
  \bibinfo{person}{S~Yu Philip}, {and} \bibinfo{person}{Zheng Wang}.}
  \bibinfo{year}{2021}\natexlab{b}.
\newblock \showarticletitle{Reinforcement learning-based dialogue guided event
  extraction to exploit argument relations}.
\newblock \bibinfo{journal}{\emph{IEEE/ACM Transactions on Audio, Speech, and
  Language Processing}}  \bibinfo{volume}{30} (\bibinfo{year}{2021}),
  \bibinfo{pages}{520--533}.
\newblock


\bibitem[Lillicrap et~al\mbox{.}(2016)]%
        {lillicrap2016continuous}
\bibfield{author}{\bibinfo{person}{Timothy~P Lillicrap},
  \bibinfo{person}{Jonathan~J Hunt}, \bibinfo{person}{Alexander Pritzel},
  \bibinfo{person}{Nicolas Heess}, \bibinfo{person}{Tom Erez},
  \bibinfo{person}{Yuval Tassa}, \bibinfo{person}{David Silver}, {and}
  \bibinfo{person}{Daan Wierstra}.} \bibinfo{year}{2016}\natexlab{}.
\newblock \showarticletitle{Continuous control with deep reinforcement
  learning.}. In \bibinfo{booktitle}{\emph{International Conference on Learning
  Representations (ICLR)}}.
\newblock


\bibitem[Lin et~al\mbox{.}(2021)]%
        {lin2021continual}
\bibfield{author}{\bibinfo{person}{Zhiwei Lin}, \bibinfo{person}{Yongtao Wang},
  {and} \bibinfo{person}{Hongxiang Lin}.} \bibinfo{year}{2021}\natexlab{}.
\newblock \showarticletitle{Continual Contrastive Self-supervised Learning for
  Image Classification}.
\newblock \bibinfo{journal}{\emph{arXiv preprint arXiv:2107.01776}}
  (\bibinfo{year}{2021}).
\newblock


\bibitem[Liu et~al\mbox{.}(2021)]%
        {liu2021self}
\bibfield{author}{\bibinfo{person}{Xiao Liu}, \bibinfo{person}{Fanjin Zhang},
  \bibinfo{person}{Zhenyu Hou}, \bibinfo{person}{Li Mian},
  \bibinfo{person}{Zhaoyu Wang}, \bibinfo{person}{Jing Zhang}, {and}
  \bibinfo{person}{Jie Tang}.} \bibinfo{year}{2021}\natexlab{}.
\newblock \showarticletitle{Self-supervised learning: Generative or
  contrastive}.
\newblock \bibinfo{journal}{\emph{IEEE Transactions on Knowledge and Data
  Engineering (TKDE)}} (\bibinfo{year}{2021}).
\newblock


\bibitem[Logeswaran and Lee(2018)]%
        {logeswaran2018efficient}
\bibfield{author}{\bibinfo{person}{Lajanugen Logeswaran} {and}
  \bibinfo{person}{Honglak Lee}.} \bibinfo{year}{2018}\natexlab{}.
\newblock \showarticletitle{An efficient framework for learning sentence
  representations}. In \bibinfo{booktitle}{\emph{International Conference on
  Learning Representations (ICLR)}}.
\newblock


\bibitem[Mai et~al\mbox{.}(2021)]%
        {mai2021supervised}
\bibfield{author}{\bibinfo{person}{Zheda Mai}, \bibinfo{person}{Ruiwen Li},
  \bibinfo{person}{Hyunwoo Kim}, {and} \bibinfo{person}{Scott Sanner}.}
  \bibinfo{year}{2021}\natexlab{}.
\newblock \showarticletitle{Supervised contrastive replay: Revisiting the
  nearest class mean classifier in online class-incremental continual
  learning}. In \bibinfo{booktitle}{\emph{Proceedings of the IEEE/CVF
  Conference on Computer Vision and Pattern Recognition (CVPR)}}.
  \bibinfo{pages}{3589--3599}.
\newblock


\bibitem[McCloskey and Cohen(1989)]%
        {mccloskey1989catastrophic}
\bibfield{author}{\bibinfo{person}{Michael McCloskey} {and}
  \bibinfo{person}{Neal~J Cohen}.} \bibinfo{year}{1989}\natexlab{}.
\newblock \showarticletitle{Catastrophic interference in connectionist
  networks: The sequential learning problem}.
\newblock In \bibinfo{booktitle}{\emph{Psychology of learning and motivation}}.
  Vol.~\bibinfo{volume}{24}. \bibinfo{publisher}{Elsevier},
  \bibinfo{pages}{109--165}.
\newblock


\bibitem[Mermillod et~al\mbox{.}(2013)]%
        {mermillod2013stability}
\bibfield{author}{\bibinfo{person}{Martial Mermillod},
  \bibinfo{person}{Aur{\'e}lia Bugaiska}, {and} \bibinfo{person}{Patrick
  Bonin}.} \bibinfo{year}{2013}\natexlab{}.
\newblock \showarticletitle{The stability-plasticity dilemma: Investigating the
  continuum from catastrophic forgetting to age-limited learning effects}.
\newblock \bibinfo{journal}{\emph{Frontiers in psychology}}
  \bibinfo{volume}{4} (\bibinfo{year}{2013}), \bibinfo{pages}{504}.
\newblock


\bibitem[Mnih et~al\mbox{.}(2013)]%
        {mnih2013playing}
\bibfield{author}{\bibinfo{person}{Volodymyr Mnih}, \bibinfo{person}{Koray
  Kavukcuoglu}, \bibinfo{person}{David Silver}, \bibinfo{person}{Alex Graves},
  \bibinfo{person}{Ioannis Antonoglou}, \bibinfo{person}{Daan Wierstra}, {and}
  \bibinfo{person}{Martin Riedmiller}.} \bibinfo{year}{2013}\natexlab{}.
\newblock \showarticletitle{Playing atari with deep reinforcement learning}.
\newblock \bibinfo{journal}{\emph{arXiv preprint arXiv:1312.5602}}
  (\bibinfo{year}{2013}).
\newblock


\bibitem[Morris et~al\mbox{.}(2020)]%
        {morris2020tudataset}
\bibfield{author}{\bibinfo{person}{Christopher Morris}, \bibinfo{person}{Nils~M
  Kriege}, \bibinfo{person}{Franka Bause}, \bibinfo{person}{Kristian Kersting},
  \bibinfo{person}{Petra Mutzel}, {and} \bibinfo{person}{Marion Neumann}.}
  \bibinfo{year}{2020}\natexlab{}.
\newblock \showarticletitle{Tudataset: A collection of benchmark datasets for
  learning with graphs}.
\newblock \bibinfo{journal}{\emph{arXiv preprint arXiv:2007.08663}}
  (\bibinfo{year}{2020}).
\newblock


\bibitem[Peng et~al\mbox{.}(2017)]%
        {peng2017incrementally}
\bibfield{author}{\bibinfo{person}{Hao Peng}, \bibinfo{person}{Jianxin Li},
  \bibinfo{person}{Yangqiu Song}, {and} \bibinfo{person}{Yaopeng Liu}.}
  \bibinfo{year}{2017}\natexlab{}.
\newblock \showarticletitle{Incrementally learning the hierarchical softmax
  function for neural language models}. In
  \bibinfo{booktitle}{\emph{Proceedings of the AAAI Conference on Artificial
  Intelligence}}, Vol.~\bibinfo{volume}{31}.
\newblock


\bibitem[Peng et~al\mbox{.}(2021)]%
        {peng2021lime}
\bibfield{author}{\bibinfo{person}{Hao Peng}, \bibinfo{person}{Renyu Yang},
  \bibinfo{person}{Zheng Wang}, \bibinfo{person}{Jianxin Li},
  \bibinfo{person}{Lifang He}, \bibinfo{person}{S~Yu Philip},
  \bibinfo{person}{Albert~Y Zomaya}, {and} \bibinfo{person}{Rajiv Ranjan}.}
  \bibinfo{year}{2021}\natexlab{}.
\newblock \showarticletitle{Lime: Low-cost and incremental learning for dynamic
  heterogeneous information networks}.
\newblock \bibinfo{journal}{\emph{IEEE Trans. Comput.}} \bibinfo{volume}{71},
  \bibinfo{number}{3} (\bibinfo{year}{2021}), \bibinfo{pages}{628--642}.
\newblock


\bibitem[Qiu et~al\mbox{.}(2020)]%
        {qiu2020gcc}
\bibfield{author}{\bibinfo{person}{Jiezhong Qiu}, \bibinfo{person}{Qibin Chen},
  \bibinfo{person}{Yuxiao Dong}, \bibinfo{person}{Jing Zhang},
  \bibinfo{person}{Hongxia Yang}, \bibinfo{person}{Ming Ding},
  \bibinfo{person}{Kuansan Wang}, {and} \bibinfo{person}{Jie Tang}.}
  \bibinfo{year}{2020}\natexlab{}.
\newblock \showarticletitle{Gcc: Graph contrastive coding for graph neural
  network pre-training}. In \bibinfo{booktitle}{\emph{Proceedings of the 26th
  ACM SIGKDD International Conference on Knowledge Discovery \& Data Mining
  (KDD)}}. \bibinfo{pages}{1150--1160}.
\newblock


\bibitem[Rebuffi et~al\mbox{.}(2017)]%
        {rebuffi2017icarl}
\bibfield{author}{\bibinfo{person}{Sylvestre-Alvise Rebuffi},
  \bibinfo{person}{Alexander Kolesnikov}, \bibinfo{person}{Georg Sperl}, {and}
  \bibinfo{person}{Christoph~H Lampert}.} \bibinfo{year}{2017}\natexlab{}.
\newblock \showarticletitle{icarl: Incremental classifier and representation
  learning}. In \bibinfo{booktitle}{\emph{Proceedings of the IEEE conference on
  Computer Vision and Pattern Recognition (CVPR)}}.
  \bibinfo{pages}{2001--2010}.
\newblock


\bibitem[Serra et~al\mbox{.}(2018)]%
        {serra2018overcoming}
\bibfield{author}{\bibinfo{person}{Joan Serra}, \bibinfo{person}{Didac Suris},
  \bibinfo{person}{Marius Miron}, {and} \bibinfo{person}{Alexandros
  Karatzoglou}.} \bibinfo{year}{2018}\natexlab{}.
\newblock \showarticletitle{Overcoming catastrophic forgetting with hard
  attention to the task}. In \bibinfo{booktitle}{\emph{International Conference
  on Machine Learning (ICML)}}. PMLR, \bibinfo{pages}{4548--4557}.
\newblock


\bibitem[Shapiro(2003)]%
        {shapiro2003monte}
\bibfield{author}{\bibinfo{person}{Alexander Shapiro}.}
  \bibinfo{year}{2003}\natexlab{}.
\newblock \showarticletitle{Monte Carlo sampling methods}.
\newblock \bibinfo{journal}{\emph{Handbooks in operations research and
  management science}}  \bibinfo{volume}{10} (\bibinfo{year}{2003}),
  \bibinfo{pages}{353--425}.
\newblock


\bibitem[Silver et~al\mbox{.}(2014)]%
        {silver2014deterministic}
\bibfield{author}{\bibinfo{person}{David Silver}, \bibinfo{person}{Guy Lever},
  \bibinfo{person}{Nicolas Heess}, \bibinfo{person}{Thomas Degris},
  \bibinfo{person}{Daan Wierstra}, {and} \bibinfo{person}{Martin Riedmiller}.}
  \bibinfo{year}{2014}\natexlab{}.
\newblock \showarticletitle{Deterministic policy gradient algorithms}. In
  \bibinfo{booktitle}{\emph{International conference on machine learning
  (ICML)}}. PMLR, \bibinfo{pages}{387--395}.
\newblock


\bibitem[Sun et~al\mbox{.}(2022)]%
        {sun2022graph}
\bibfield{author}{\bibinfo{person}{Qingyun Sun}, \bibinfo{person}{Jianxin Li},
  \bibinfo{person}{Hao Peng}, \bibinfo{person}{Jia Wu},
  \bibinfo{person}{Xingcheng Fu}, \bibinfo{person}{Cheng Ji}, {and}
  \bibinfo{person}{S~Yu Philip}.} \bibinfo{year}{2022}\natexlab{}.
\newblock \showarticletitle{Graph structure learning with variational
  information bottleneck}. In \bibinfo{booktitle}{\emph{Proceedings of the AAAI
  Conference on Artificial Intelligence}}, Vol.~\bibinfo{volume}{36}.
  \bibinfo{pages}{4165--4174}.
\newblock


\bibitem[Sun et~al\mbox{.}(2021)]%
        {sun2021sugar}
\bibfield{author}{\bibinfo{person}{Qingyun Sun}, \bibinfo{person}{Jianxin Li},
  \bibinfo{person}{Hao Peng}, \bibinfo{person}{Jia Wu},
  \bibinfo{person}{Yuanxing Ning}, \bibinfo{person}{Philip~S Yu}, {and}
  \bibinfo{person}{Lifang He}.} \bibinfo{year}{2021}\natexlab{}.
\newblock \showarticletitle{Sugar: Subgraph neural network with reinforcement
  pooling and self-supervised mutual information mechanism}. In
  \bibinfo{booktitle}{\emph{Proceedings of the Web Conference 2021}}.
  \bibinfo{pages}{2081--2091}.
\newblock


\bibitem[Sutton et~al\mbox{.}(1999)]%
        {sutton1999policy}
\bibfield{author}{\bibinfo{person}{Richard~S Sutton}, \bibinfo{person}{David
  McAllester}, \bibinfo{person}{Satinder Singh}, {and} \bibinfo{person}{Yishay
  Mansour}.} \bibinfo{year}{1999}\natexlab{}.
\newblock \showarticletitle{Policy gradient methods for reinforcement learning
  with function approximation}.
\newblock \bibinfo{journal}{\emph{Advances in neural information processing
  systems (NeurIPS)}}  \bibinfo{volume}{12} (\bibinfo{year}{1999}).
\newblock


\bibitem[Uhlenbeck and Ornstein(1930)]%
        {uhlenbeck1930theory}
\bibfield{author}{\bibinfo{person}{George~E Uhlenbeck} {and}
  \bibinfo{person}{Leonard~S Ornstein}.} \bibinfo{year}{1930}\natexlab{}.
\newblock \showarticletitle{On the theory of the Brownian motion}.
\newblock \bibinfo{journal}{\emph{Physical review}} \bibinfo{volume}{36},
  \bibinfo{number}{5} (\bibinfo{year}{1930}), \bibinfo{pages}{823}.
\newblock


\bibitem[Van~den Oord et~al\mbox{.}(2018)]%
        {oord2018representation}
\bibfield{author}{\bibinfo{person}{Aaron Van~den Oord}, \bibinfo{person}{Yazhe
  Li}, {and} \bibinfo{person}{Oriol Vinyals}.} \bibinfo{year}{2018}\natexlab{}.
\newblock \showarticletitle{Representation learning with contrastive predictive
  coding}.
\newblock \bibinfo{journal}{\emph{arXiv e-prints}} (\bibinfo{year}{2018}),
  \bibinfo{pages}{arXiv--1807}.
\newblock


\bibitem[Wu et~al\mbox{.}(2018)]%
        {wu2018unsupervised}
\bibfield{author}{\bibinfo{person}{Zhirong Wu}, \bibinfo{person}{Yuanjun
  Xiong}, \bibinfo{person}{Stella~X Yu}, {and} \bibinfo{person}{Dahua Lin}.}
  \bibinfo{year}{2018}\natexlab{}.
\newblock \showarticletitle{Unsupervised feature learning via non-parametric
  instance discrimination}. In \bibinfo{booktitle}{\emph{Proceedings of the
  IEEE conference on computer vision and pattern recognition (CVPR)}}.
  \bibinfo{pages}{3733--3742}.
\newblock


\bibitem[Xu et~al\mbox{.}(2017)]%
        {xu2017reinforcement}
\bibfield{author}{\bibinfo{person}{Chang Xu}, \bibinfo{person}{Tao Qin},
  \bibinfo{person}{Gang Wang}, {and} \bibinfo{person}{Tie-Yan Liu}.}
  \bibinfo{year}{2017}\natexlab{}.
\newblock \showarticletitle{Reinforcement learning for learning rate control}.
\newblock \bibinfo{journal}{\emph{arXiv preprint arXiv:1705.11159}}
  (\bibinfo{year}{2017}).
\newblock


\bibitem[You et~al\mbox{.}(2020)]%
        {you2020graph}
\bibfield{author}{\bibinfo{person}{Yuning You}, \bibinfo{person}{Tianlong
  Chen}, \bibinfo{person}{Yongduo Sui}, \bibinfo{person}{Ting Chen},
  \bibinfo{person}{Zhangyang Wang}, {and} \bibinfo{person}{Yang Shen}.}
  \bibinfo{year}{2020}\natexlab{}.
\newblock \showarticletitle{Graph contrastive learning with augmentations}.
\newblock \bibinfo{journal}{\emph{Advances in Neural Information Processing
  Systems (NeurIPS)}}  \bibinfo{volume}{33} (\bibinfo{year}{2020}),
  \bibinfo{pages}{5812--5823}.
\newblock


\end{thebibliography}


\end{document}